\documentclass[twoside]{article}

\usepackage[accepted]{aistats2020}
\usepackage[utf8]{inputenc} 
\usepackage[T1]{fontenc}    
\usepackage[colorlinks=true,linkcolor=blue,citecolor=cyan]{hyperref}       
\usepackage{url}            
\usepackage{booktabs}       
\usepackage{amsfonts}       
\usepackage{nicefrac}       
\usepackage{microtype}      

\usepackage{verbatim}

\usepackage{graphicx}
\usepackage{amsfonts}
\usepackage{amsthm}
\usepackage{amsmath}
\usepackage{amscd}
\usepackage{parskip}
\usepackage{enumitem}
\usepackage{verbatim}
\usepackage{bibentry}
\usepackage{natbib}
\usepackage{booktabs}
\usepackage{thm-restate}

\usepackage{bbm}

\usepackage{subcaption}

\newcommand{\eqdef}{\stackrel{\mbox{\tiny def}}{=}}

\newtheoremstyle{thmstyle}
  {5pt} 
  {\topsep} 
  {} 
  {} 
  {\bfseries} 
  {.} 
  {.5em} 
  {} 
\theoremstyle{thmstyle}

\newtheorem{theorem}{Theorem}[section]

\newtheorem{definition}[theorem]{Definition}
\newtheorem{proposition}[theorem]{Proposition}

\usepackage{wrapfig}

\usepackage{algorithm}
\usepackage{algorithmic}

\newcommand{\beq}{\begin{equation}}
\newcommand{\eeq}{\end{equation}}

\newcommand{\beqa}{\begin{eqnarray}}
\newcommand{\eeqa}{\end{eqnarray}}

\newcommand{\beqan}{\begin{eqnarray*}}
\newcommand{\eeqan}{\end{eqnarray*}}

\usepackage{selectp}

\newcommand{\statespace}{\mathcal{X}}
\newcommand{\actionspace}{\mathcal{A}}

\newcommand{\action}{a}

\newcommand{\ctrace}{C-trace}
\newcommand{\alpharetrace}{$\alpha$-Retrace}
\newcommand{\alphatreebackup}{$\alpha$-TreeBackup}
\newcommand{\alpharetraceone}{$1$-Retrace}
\newcommand{\alpharetraceContract}{C}

\usepackage{tikz-cd}

\usepackage{mathtools}

\usepackage{mathrsfs}

\begin{document}

%

%

\twocolumn[

\aistatstitle{Adaptive Trade-Offs in Off-Policy Learning}

\aistatsauthor{Mark Rowland* \And Will Dabney* \And R\'emi Munos}

\aistatsaddress{DeepMind \And DeepMind \And DeepMind} ]

\begin{abstract}
  A great variety of off-policy learning algorithms exist in the literature, and new breakthroughs in this area continue to be made, improving theoretical understanding and yielding state-of-the-art reinforcement learning algorithms. In this paper, we take a unifying view of this space of algorithms, and consider their trade-offs of three fundamental quantities: \emph{update variance}, \emph{fixed-point bias}, and \emph{contraction rate}. This leads to new perspectives on existing methods, and also naturally yields novel algorithms for off-policy evaluation and control. We develop one such algorithm, \ctrace, demonstrating that it is able to more efficiently make these trade-offs than existing methods in use, and that it can be scaled to yield state-of-the-art performance in large-scale environments.
\end{abstract}

\section{Introduction}\label{sec:intro}

Off-policy learning is crucial to modern reinforcement learning, allowing agents to learn from memorised data, demonstrations, and exploratory behaviour \citep{szepesvari2010algorithms,sutton2018reinforcement}. As such, it is a long-studied problem, with a variety of well-understood associated algorithms; see \citep{precup2000eligibility,kakade2002approximately,dudik2014doubly,thomas2016data,munos2016safe,mahmood2017multi,farajtabar2018more} for a representative selection of publications.

However, this paper is motivated by the observation that in spite of this theoretical progress, several state-of-the-art value-based reinforcement learning agents (notably Rainbow \citep{hessel2018rainbow} and R2D2 \citep{kapturowski2018recurrent}) eschew these off-policy algorithms, attaining better performance by using \emph{uncorrected} returns, which do not account for the fact that data is generated off-policy. Further research has corroborated this observation \citep{hernandez2019understanding}. This raises two central research questions: (i) How can we understand the strong performance of uncorrected returns? (ii) Can we distil these advantages, and combine them with existing off-policy algorithms to improve their performance?

One of the principal contributions of this paper is to show that the performance of all off-policy evaluation algorithms can be decomposed into three fundamental quantities: \emph{contraction rate}, \emph{fixed-point bias}, and \emph{variance}; see Figure~\ref{fig:first-tradeoff} for a preliminary illustration. 
Intuitively, \emph{fixed-point bias} describes the error of an algorithm in the limit of infinite data, \emph{contraction rate} describes the speed at which an algorithm approaches its infinite-data limit, and \emph{variance} describes to what extent randomly observed data can perturb the algorithm.

This decomposition yields an interpretation of the empirical success of uncorrected returns, and an answer to question (i) above; namely, that they are efficiently making a trade-off between fixed-point bias and the other fundamental quantities. Further, this suggests an answer to question (ii) --- that we may be able to improve existing off-policy algorithms by incorporating a means of making such a trade-off. This leads us to the development of \emph{C-trace}, a new off-policy algorithm that achieves strong empirical performance in several large-scale environments.

We develop the trade-off framework mentioned above in Section \ref{sec:contractionbiasvariance}, proving the existence of the three fundamental quantities described above, and showing that all off-policy algorithms necessarily make an implicit trade-off between these quantities. We then use this framework to develop new off-policy learning algorithms, \alpharetrace\ and C-trace, in Section \ref{sec:retrace-alpha}, and study its contraction and convergence properties. Finally, we demonstrate their empirical effectiveness in tabular domains and when applied to two deep reinforcement learning agents, DQN \citep{mnih2015human} and R2D2 \citep{kapturowski2018recurrent}, in Section \ref{sec:large-experiments}.

\begin{figure*}
	\centering
	\includegraphics[keepaspectratio,width=.9\textwidth]{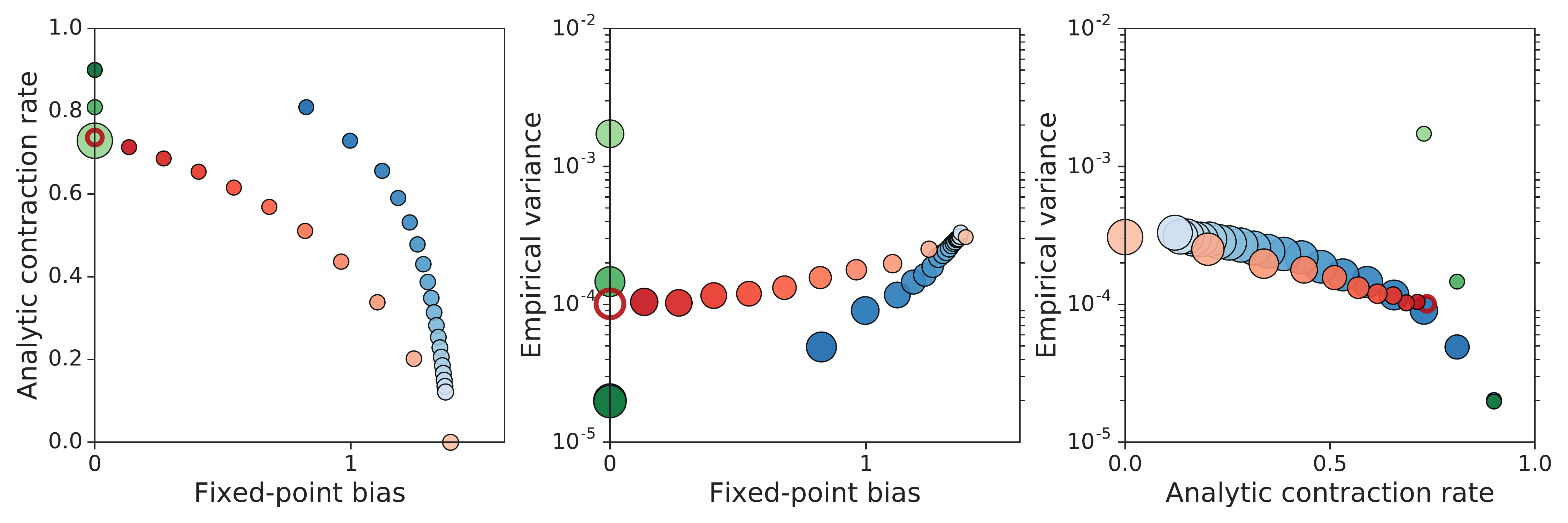}
	\caption{Trade-offs made by $n$-step uncorrected returns (dark blue [$n=1$] through to light blue [$n=20$]), $n$-step importance corrected returns (dark green [$n=1$] through to light green [$n=3$]), Retrace (red open circle). Also pictured is the new method \alpharetrace\ (dark red [$\alpha = 1$] through to light red [$\alpha=0$]), introduced in Section \ref{sec:retrace-alpha}. All quantities are calculated for a fixed evaluation problem in a small, randomly generated MDP; see Appendix Section \ref{sec:envs} for further environment details. In each plot, the magnitude of the points illustrates the relative scale of the third trade-off variable.}
	\label{fig:first-tradeoff}
\end{figure*}

\subsection{Notation and preliminary definitions}

Throughout, we consider a Markov decision process (MDP) with finite state space $\statespace$, finite action space $\actionspace$, discount factor $\gamma \in [0, 1)$, transition kernel $P: \statespace \times \actionspace \rightarrow \mathscr{P}(\statespace)$, reward distributions $\mathcal{R}: \statespace \times \actionspace \rightarrow \mathscr{P}(\mathbb{R})$, and some initial state distribution $\nu_0 \in \mathscr{P}(\statespace)$. Given a Markov policy $\pi : \statespace \rightarrow \mathscr{P}(\mathcal{A})$, we write $(X_t, A_t, R_t)_{t\geq0}$ for the process describing the sequence of states visited, actions taken, and rewards received by an agent acting in the MDP according to $\pi$, so that $R_t |X_t,A_t \sim \mathcal{R}(X_t, A_t)$ for all $t \geq 0$. Additionally, we write $r(x, a)$ for the expected immediate reward received after taking action $a$ in state $x$. Given a policy $\pi$, the task of \emph{evaluation} is to learn the function $Q^\pi(x, a) = \mathbb{E}_\pi\left\lbrack \sum_{t=0}^\infty \gamma^t R_t | X_0 = x, A_0 = a\right\rbrack$, where $\mathbb{E}_\pi$ denotes expectation with respect to the distribution over trajectories induced by $\pi$. The task of \emph{control} is to identify the Markov policy $\pi^*$ maximising the quantity $\mathbb{E}\left\lbrack Q^\pi(X_0, A_0) \right\rbrack$, where $A_0 \sim \pi(\cdot | X_0)$, and $X_0 \sim \nu_0$. We also define the one-step evaluation operator $T^\mu : \mathbb{R}^{\statespace \times \actionspace} \rightarrow \mathbb{R}^{\statespace \times \actionspace}$ associated with a Markov policy $\mu : \statespace \rightarrow \mathscr{P}(\actionspace)$ by
\begin{align}
    &(T^\mu Q) (x, a) = \\
    &r(x, a) + \gamma \sum_{x^\prime \in \statespace\, , a^\prime \in \actionspace} P(x^\prime | x, a) \mu(a^\prime | x^\prime) Q(x^\prime, a^\prime) \nonumber \, ,
\end{align}
for all $Q \in \mathbb{R}^{\statespace\times\actionspace}$, and $(x, a) \in \statespace\times\actionspace$.

We now briefly give formal definitions of the key concepts we seek to analyse in this paper.

\begin{definition}
    An \textbf{evaluation update rule} for evaluating a policy $\pi$ under a behaviour policy $\mu$ is a function $\hat{T} : \mathbb{R}^{\statespace \times \actionspace} \times (\statespace \times \actionspace \times \mathbb{R})^* \rightarrow \mathbb{R}$ that takes as input a value function estimate $Q$ and a trajectory $(x_t, a_t, r_t)_{t\geq0}$ given by following $\mu$, and outputs an update for $Q(x_0, a_0)$. There is an associated \textbf{evaluation operator} $T: \mathbb{R}^{\statespace \times \actionspace} \rightarrow \mathbb{R}^{\statespace \times \actionspace}$, given by
    \begin{align*}
        (TQ) (x, a)\!\! =\!\! \mathbb{E}_\mu\!\left\lbrack \hat{T}(Q, (X_t, A_t, R_t)_{t=0}^\infty) \middle| X_0 = x, A_0 = a \right\rbrack \, ,
    \end{align*}
    for all $Q \in \mathbb{R}^{\statespace\times\action}$ and $(x, a) \in \statespace\times\actionspace$.
\end{definition}
\begin{definition}
    The \textbf{contraction rate} of an operator $T : \mathbb{R}^{\statespace \times \actionspace} \rightarrow \mathbb{R}^{\statespace \times \actionspace}$ is given by
    \begin{align*}
        \Gamma = \sup_{\substack{Q, Q^\prime \in \mathbb{R}^{\statespace \times \actionspace} \\ Q \not= Q^\prime}} \|TQ - TQ^\prime\|_\infty / \|Q - Q^\prime\|_\infty \, ,
    \end{align*}
    An operator is said to be \emph{contractive} if $\Gamma < 1$. We can also consider state-action contraction rates via the quantities $\sup_{Q \not= Q^\prime} |(TQ)(x, a) - (TQ^\prime)(x, a)|/\|Q - Q^\prime\|_\infty$. 
\end{definition}
\begin{definition}
    For a contractive operator $T$ targeting a policy $\pi$, the \textbf{fixed-point bias} of $T$ is given by $\|Q^\pi - \hat{Q}^\pi\|_2$, where $\hat{Q}^\pi$ is the unique fixed point of $T$ (guaranteed to exist by the contractivity of $T$).
\end{definition}
\begin{definition}
    The \textbf{variance} of an update rule $\hat{T}$ stochastically approximating an operator $T$ at approximate value function $Q$ and an initial state-action distribution $\nu \in \mathscr{P}(\statespace\times\actionspace)$ is $\mathbb{E}_{(X_0, A_0) \sim \nu}\!\left\lbrack \mathbb{E}_\mu\!\left\lbrack \|\hat{T}(Q, (X_t, A_t, R_t)_{t=0}^\infty)\! -\! TQ \|_2^2 \middle| X_0, A_0 \right\rbrack \right\rbrack$.
\end{definition}
\section{Contraction, bias, and variance}\label{sec:contractionbiasvariance}
We begin with two motivating examples from recent research in off-policy evaluation methods, illustrating examples of the types of trade-offs we seek to describe in this paper.

\textbf{$n$-step uncorrected returns.} Recently proposed agents such as Rainbow \citep{hessel2018rainbow} and R2D2 \citep{kapturowski2018recurrent} have made use of the \emph{uncorrected }$n$\emph{-step return} in constructing off-policy learning algorithms. Consistent with these results, \citet{hernandez2019understanding} observed that these uncorrected updates frequently outperformed off-policy corrections. Given an estimate $\hat{Q}$ of the action-value function $Q^\pi$, the $n$-step uncorrected target for $\hat{Q}(x_0, a_0)$, given a trajectory $(x_0, a_0, r_0, x_1, a_1, r_1,\ldots, x_n)$ of experience generated according to behaviour policy $\mu$, is given by
\begin{align}
    \sum_{s=0}^{n-1} \gamma^{s} r_{s} + \gamma^n \mathbb{E}_{A \sim \pi(\cdot|x_{n})}\left\lbrack \hat{Q}(x_{n}, A)\right\rbrack \, .
\end{align}
The adjective \emph{uncorrected} contrasts this update target against the $n$\emph{-step importance-weighted return} target, which takes the following form:
\begin{align}
    \sum_{s=0}^{n-1} \rho_{1:s} \gamma^s r_s + \rho_{1:n-1} \gamma^n \mathbb{E}_{A \sim \pi(\cdot|x_{n})}\!\!\left\lbrack \hat{Q}(x_{n}, A)\right\rbrack \, ,
\end{align}
where we write $\rho_t = \pi(a_t|x_t)/\mu(a_t|x_t)$, and $\rho_{s:t} = \prod_{u=s}^t \rho_u$ for each $1\leq s \leq t$. 
Empirically, the former has been observed to work very well in these recent works, whilst the latter is often too unstable to be used; this fact is often attributed to the \emph{high variance} of the importance-weighted update, with the uncorrected update having relatively low variance by comparison. We also observe that the uncorrected update is a stochastic approximation to the operator $(T^\mu)^{n-1}T^\pi$, whilst the importance-weighted update is a stochastic approximation to $(T^\pi)^n$. From this, it follows that under usual stochastic approximation conditions, a sequence of importance-weighted updates will converge to the true action-function $Q^\pi$ associated with $\pi$, whilst the uncorrected updates will converge to the value function of the time-inhomogeneous policy that follows $\pi$ for a single step, followed by $n-1$ steps of $\mu$, and then repeats; see Proposition \ref{prop:fixedpoints} in Appendix Section \ref{sec:appendixMoreResults} for further explanation.

The above discussion shows that we may view the use of uncorrected returns as trading off update \emph{variance} for \emph{accuracy of the operator fixed point}; an example of the classical bias-variance trade-off in statistics and machine learning, albeit in the context of fixed-point iteration.

\textbf{Retrace. } \citet{munos2016safe} proposed an off-policy evaluation update target, Retrace, given in its forward-view version by
\begin{align}\label{eq:retrace}
    \hat{Q}(x_0, a_0) \!+\! \sum_{s \geq 0} \bar{\rho}_{1:s} \gamma^s \Delta_s \, ,
\end{align}
where we write $\bar{\rho}_t = \min(1,\rho_t)$, and $\bar{\rho}_{s:t} = \prod_{u=s}^t \bar{\rho}_u$ for each $1 \leq s \leq t$, and define the temporal difference (TD) error $\Delta_s$ by
\begin{align*}
    \Delta_s \eqdef r_{s} + \gamma \mathbb{E}_{A \sim \pi(\cdot|x_{s+1})}\!\left\lbrack \hat{Q}(x_{s+1}, A) \right\rbrack\! - \hat{Q}(x_{s}, a_{s}) \, .
\end{align*}
By clipping the importance weights associated with each TD error, the variance associated with the update rule is reduced relative to importance-weighted returns, whilst no bias is introduced; the fixed point of the associated Retrace operator remains the true action-value function $Q^\pi$. However, the clipping of the importance weights effectively \emph{cuts the traces} in the update, resulting in the update placing less weight on later TD errors, and thus worsening the contraction rate of the corresponding operator. Thus, Retrace can be interpreted as trading off a \emph{reduction in update variance} for a \emph{larger contraction rate}, relative to importance-weighted $n$-step returns.

We discuss more examples of off-policy learning algorithms in Section \ref{sec:relatedwork}. We also note that $\lambda$-variants of the algorithms described above also exist; for clarity and conciseness, we limit our exposition to the case $\lambda=1$ in the main paper, noting that the results straightforwardly extend to $\lambda \in (0,1)$.

We now briefly return to Figure \ref{fig:first-tradeoff}, which quantitatively illustrates the trade-offs discussed above. We highlight several interesting observations. Whilst all importance-weighted updates have no fixed-point bias, their variance grows exceptionally quickly with $n$. Retrace manages to achieve a similar contraction rate to the $3$-step importance-weighted update, but without incurring high variance. Our new algorithm, \alpharetrace, appears to be Pareto efficient relative to the $n$-step uncorrected methods in the left-most plot; for any contraction rate that an $n$-step uncorrected method achieves, there is a value of $\alpha$ such that \alpharetrace\ achieves this contraction rate whilst incurring less fixed-point bias; this is corroborated by further empirical results in Appendix Section \ref{sec:appendixMoreResults}.

\subsection{Downstream tasks and bounds}\label{sec:downstreamtasks}
Whilst the trade-offs at the level of individual updates described above are straightforward to describe, in reinforcement learning we are ultimately interested in one of two problems, either \emph{evaluation} or \emph{control}, defined formally below.

\textbf{The evaluation problem. } Given a target policy $\pi$, a budget of experience generated from a behaviour policy $\mu$, and a computational budget, compute an accurate approximation $\hat{Q}$ to $Q^\pi$, in the sense of incurring low error $\|\hat{Q} - Q^\pi\|$, for some norm $\|\cdot \|$.

\textbf{The control problem. } Given a budget of experience and computation, find a policy $\pi$ such that expected return under $\pi$ is maximised.

It is intuitively clear that for each of these problems, an evaluation scheme with low contraction rate, low update variance, and low fixed-point bias is advantageous, but no update is known to possess all three of these attributes simultaneously. What is less clear is how these three properties should be traded off against one another in designing an efficient off-policy learning algorithm. For example, how much fixed-point accuracy should one be willing trade off in exchange for a halved update variance? Such questions, in general, have complicated dependence on the precise update rule, policies in question, and environment, and so it appears unlikely that too much progress can be made in great generality. However, we can provide some some insights from understanding this fundamental trade-off.

\begin{restatable}{proposition}{propDecomposition}\label{prop:decomposition}
    Consider the task of evaluation of a policy $\pi$ under behaviour $\mu$, and consider an update rule $\hat{T}$ which stochastically approximates the application of an operator $\widetilde{T}$, with contraction rate $\Gamma$ and fixed point $\widetilde{Q}$, to an initial estimate $Q$. Then we have the following decomposition:
    \begin{align*}
        &\mathbb{E}\left\lbrack \| \hat{T}Q - Q^\pi \|_\infty \right\rbrack \leq \\
        &\underbrace{\mathbb{E}\left\lbrack \| \hat{T}Q - \widetilde{T}Q \|^2_2 \right\rbrack^{1/2}}_{\text{(Root) variance}} + \underbrace{\Gamma\| Q - \widetilde{Q} \|_\infty}_{\text{Contraction}} + \underbrace{\| \widetilde{Q} - Q^\pi \|_2}_{\text{Fixed-point bias}} \, .
    \end{align*}
\end{restatable}
Note that $\hat{T}$ is arbitrary, and may, for example, represent the $n$-fold application of a simpler operator. 
This result gives some sense of how these trade-offs feed into evaluation quality; related decompositions are also possible (see Appendix Section \ref{sec:appendixMoreResults}). We next show that there really is a trade-off to be made, in the sense that it is not possible for an update based on limited data to simultaneously have low variance, contraction rate, and fixed-point bias across a range of MDPs.
\begin{restatable}{theorem}{thmLowerBound}\label{thm:lowerbound}
    Consider an update rule $\hat{T}$ with corresponding operator $T$, and consider the collection $\mathcal{M}=\mathcal{M}(\statespace, \actionspace, P, \gamma, R_\text{max})$ of MDPs with common state space, action space, transition kernel, and discount factor (but varying rewards, with maximum immediate reward bounded in absolute value by $R_\text{max}$). Fix a target policy $\pi$, and a random variable $Z$, the set of transitions used by the operator $\hat{T}$; these could be transitions encountered in a trajectory following the behaviour $\mu$, or i.i.d. samples from the discounted state-action visitation distribution under $\mu$. We denote the mismatch between $\pi$ and $Z$ at level $\delta \in (0,1)$ by
    \begin{align*}
        D(Z,\pi,\delta) \eqdef \max \{ &d_{(x, a), \pi}(\Omega)\ |\ \Omega \subseteq \statespace\times\actionspace \text{ s.t.}\\
        &\mathbb{P}(Z \cap \Omega \not= \emptyset) \leq \delta\, ,  (x, a) \!\in\! \statespace\!\times\!\actionspace\} \, ,
    \end{align*}
    where $d_{(x, a), \pi}$ is the discounted state-action visitation distribution for trajectories initialised with $(x, a)$, following $\pi$ thereafter. 
    Denoting the variance, contraction rate, and fixed-point bias of $\hat{T}$ for a particular MDP $M \in \mathcal{M}$ by $\mathbb{V}(M)$, $\Gamma(M)$ and $B(M)$ respectively, we have
    \begin{align*}
        &\sup_{M \in \mathcal{M}} \left\lbrack \sqrt{\mathbb{V}(M)} + \frac{2R_\text{max}}{1-\gamma} \Gamma(M) + B(M)  \right\rbrack \geq \\
        &\qquad\qquad \sup_{\delta \in (0,1)} (1-\delta) D(Z, \pi, \delta) R_\text{max}/(1-\gamma) \, .
    \end{align*}
\end{restatable}

In addition to the above results, which we believe to be novel, there is extensive literature exploring particular aspects of these trade-offs, which we discuss further in Section \ref{sec:relatedwork}. Having made this space of trade-offs between contraction, bias, and variance explicit, a natural questions is how other update rules might be modified to exploit different parts of the space. In particular, incurring some amount of fixed-point bias for reduced variance made by $n$-step uncorrected returns in Rainbow and R2D2 is particularly effective in practice --- is there a way to introduce a similar trade-off in an algorithm with adaptive trace lengths, such as Retrace? We explore this question in the next section.
\section{New off-policy updates: \alpharetrace\ and \ctrace}\label{sec:retrace-alpha}

The Retrace update in Equation \eqref{eq:retrace} has been observed, in certain scenarios, to cut traces prematurely \citep{mahmood2017multi}; that is, using $n$-step uncorrected returns for suitable $n$ leads to a superior contraction rate relative to Retrace, outweighing the corresponding incurred bias. A natural question is how Retrace can be modified to overcome this phenomenon. In the language of Section \ref{sec:contractionbiasvariance}, is there a way that Retrace can be adapted so as to trade off contraction rate for fixed-point bias? The reduced contraction rate comes from cases where the truncated importance weights $\min(1, \pi(a_t|x_t)/\mu(a_t|x_t))$ appearing in \eqref{eq:retrace} are small, so a natural way to improve the contraction rate is to move the target policy closer towards the behaviour.

\begin{algorithm}[H]
    \caption{\alpharetrace\ for policy iteration}
    \label{alg:retracealpha-PI}
    \begin{algorithmic}
        \STATE{Initialise target policy $\widetilde{\pi}$ and behaviour $\mu$.}
        \FOR{each policy improvement round:}
        \STATE{Select $\alpha \in [0,1]$, and set new target policy $\pi = \alpha \widetilde{\pi} + (1-\alpha)\mu$.}
        \STATE{Learn $\hat{Q}^\pi : \statespace \times \actionspace \rightarrow \mathbb{R}$ via Retrace under behaviour policy $\mu$.}
        \STATE{Set $\widetilde{\pi} = \text{Greedy}(\hat{Q}^\pi)$.}
        \STATE{Set new behaviour policy $\mu$.}
    \ENDFOR
    \end{algorithmic}
\end{algorithm}
To this end, we propose \alpharetrace, a family of algorithms that applies Retrace to a target policy given by a mixture of the original target and the behaviour, thus achieving the aforementioned trade-off.
In Algorithm \ref{alg:retracealpha-PI}, we describe how \alpharetrace\ can be used within a (modified) policy iteration scheme for control.
Note that \alpharetraceone\ is simply the standard Retrace algorithm. We refer back to Figure~\ref{fig:first-tradeoff}, the left-most plot of which shows that this mixture coefficient precisely yields a trade-off between fixed-point bias and contraction rate that we sought at the end of Section~\ref{sec:contractionbiasvariance}.

The means by which $\alpha$ should be set is left open at this stage; adjusting it allows a trade-off of contraction rate and fixed-point bias. In Section \ref{sec:adaptingalpha}, we describe a stochastic approximation procedure for updating $\alpha$ online to obtain a desired contraction rate. 

\textbf{Specificity to Retrace. } Whilst the mixture target of \alpharetrace\ is a natural choice, we highlight that this choice is in fact specific to the structure of Retrace. In Appendix Section \ref{sec:further-tradeoff}, we visualise trade-offs made by 
analogous adjustment to the TreeBackup update \citep{precup2000eligibility}, showing that mixing the behaviour policy into the target simply leads to an accumulation of fixed-point bias, with limited benefits in terms of contraction rate or variance.

\subsection{Analysis}

We now provide several results describing the contraction rate of \alpharetrace\ in detail, and how the fixed-point bias introduced by $\alpha < 1$ may be useful in the case of control tasks. We begin with a preliminary definition.

\begin{definition}
    For a state-action pair $(x, a) \in \statespace\times\actionspace$, 
    Two policies $\pi_1$, $\pi_2$ are said to be $(x, a)$-\textbf{distinguishable} under a third policy $\mu$ if there exists $x^\prime \in \statespace$ in the support of the discounted state visitation distribution under $\mu$ starting from state-action pair $(x, a)$, such that $\pi_1(\cdot|x^\prime) \not= \pi_2(\cdot|x^\prime)$, and are said to be $(x, a)$-\textbf{indistinguishable} under $\mu$ otherwise.
\end{definition}

\begin{restatable}{proposition}{propContractionRate}\label{prop:contractionrate}
    The operator associated with the \alpharetrace\ update for evaluating $\pi$ given behaviour $\mu$ has a state-action-dependent contraction rate of
    \begin{align}\label{eq:retracealphacontraction}
        &\alpharetraceContract(\alpha| x, a) \eqdef 1 - (1-\gamma) \times  \\
        & \ \ \mathbb{E}_\mu\!\left\lbrack \sum_{t=0}^\infty \gamma^t \prod_{s=1}^t \left((1-\alpha) + \alpha\bar{\rho}_s \right)\middle| (X_0,A_0) = (x, a) \right\rbrack \, , \nonumber
    \end{align}
    for each $(x, a) \in \statespace\times\actionspace$. 
    Viewed as a function of $\alpha \in [0,1]$, this contraction rate is continuous, monotonically increasing, with minimal value $0$, and maximal value no greater than $\gamma$. Further, the contraction rate is \emph{strictly} monotonic iff $\pi$ and $\mu$ are $(x, a)$-distinguishable under $\mu$.
\end{restatable}

The exact contraction rate of \alpharetrace\ is thus $\sup_{(x, a) \in \statespace\times\actionspace} \alpharetraceContract(\alpha|x, a)$, which inherits the continuity and monotonicity properties of the state-action-dependent rates. Our next result motivates the use of \alpharetrace\ within control algorithms.

\begin{restatable}{proposition}{thmbettercontraction}
    Consider a target policy $\pi$, let $\mu$ be the behavioural policy, and assume that there is a unique greedy action $a^*(x) \in \actionspace$ with respect to $Q^\pi$ at each state $x$ for each $x \in \statespace$. 
    Then there exists a value of $\alpha \in (0,1)$ such that the greedy policy with respect to the fixed point of \alpharetrace\ coincides with the greedy policy with respect to $Q^\pi$, and the contraction rate for this \alpharetrace\ is no greater than that for \alpharetraceone. 
    Further, if $\pi$ and $\mu$ are $(x, a)$-distinguishable under $\mu$ for all $(x, a)\in\statespace\times\actionspace$, then the contraction rate of \alpharetrace\ is strictly lower than that of \alpharetraceone.
\end{restatable}

\begin{figure}
	\centering
	\includegraphics[keepaspectratio,width=.35\textwidth]{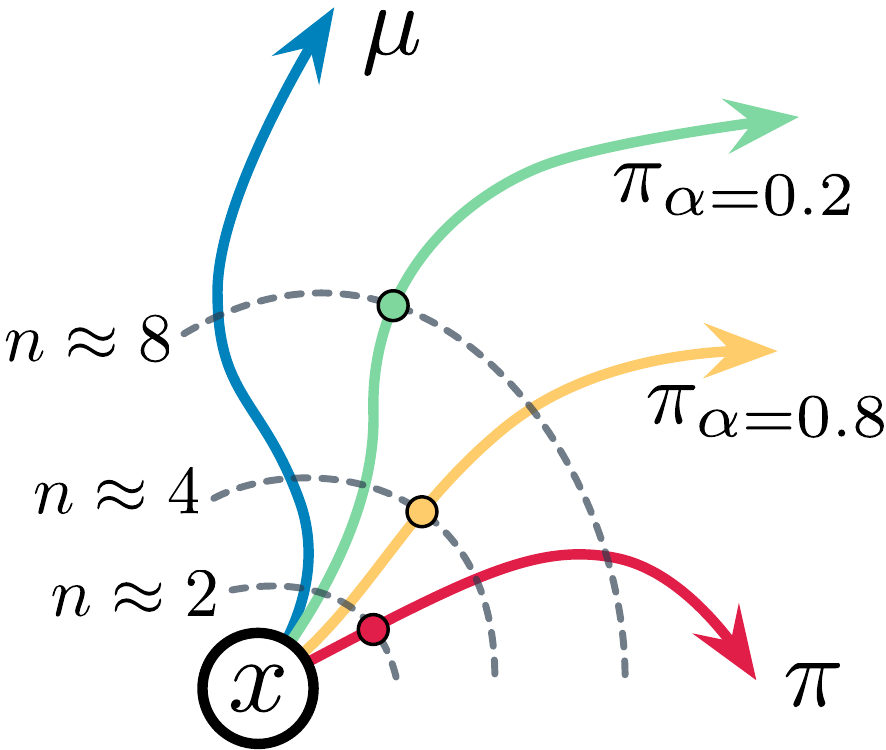}
	\caption{Interpolating between target policy $\pi$ and behaviour policy $\mu$ with $\alpha \in \{0.0, 0.2, 0.8, 1.0\}$ produces different expected trajectories shown by each coloured line. As the mixture policy more closely resembles the behaviour policy, $\alpha$-Retrace allows more off-policy data to be used (dashed line, numbers indicate expected trace-length), cuts traces (coloured points) later, yielding lower contraction rates equivalent to $n$-step methods with larger $n$.
	C-trace adapts $\alpha$ online to achieve a stable trace length throughout training.}
	\label{fig:diagram}
\end{figure}

\subsection{\ctrace: adapting $\alpha$ online}\label{sec:adaptingalpha}

An empirical shortcoming of Retrace noted earlier is its tendency to pessimistically cut traces. Adapting the mixture parameter $\alpha$ within \alpharetrace\ yields a natural way to ensure that a desired trace length (or contraction rate) is attained. In this section, we propose \ctrace, which uses \alpharetrace\ updates whilst dynamically adjusting $\alpha$ to attain a target contraction rate $\Gamma$; a schematic illustration is given in Figure~\ref{fig:diagram}.

The contraction rate $\sup_{(x, a) \in \statespace\times\actionspace} \alpharetraceContract(\alpha| x, a)$ is difficult to estimate online, so we work instead with the averaged contraction rate $\alpharetraceContract_\nu(\alpha) = \mathbb{E}_{(X, A) \sim \nu}\lbrack \alpharetraceContract(\alpha|X, A) \rbrack$, where $\nu$ is the training distribution over state-action pairs; where clear, we will drop $\nu$ from notation. 
It follows straighforwardly from Proposition \ref{prop:contractionrate} that $\alpharetraceContract_\nu(\alpha)$ is monotonic in $\alpha$. This suggests that a standard Robbins-Monro stochastic approximation update rule for $\alpha$ may be applied to guide $\alpharetraceContract_\nu(\alpha)$ towards $\Gamma$ --- we describe such a scheme below.
To avoid optimisation issues with the constraint $\alpha \in [0,1]$, we parameterise $\alpha$ as $\sigma(\phi)$, where $\sigma$ is the standard sigmoid function, and $\phi \in \mathbb{R}$ is an unconstrained real variable. For brevity, we will simply write $\alpha(\phi)$. Since $\sigma$ is monotonic and continuous, the contraction rate is still monotonic and continuous in $\phi$.

Recall from \eqref{eq:retracealphacontraction} that the contraction rate $\alpharetraceContract(\alpha|x, a)$ of the \alpharetrace\ operator with target $\pi$ and behaviour $\mu$ can be expressed as an expectation over trajectories following $\mu$, and thus can be unbiasedly approximated using such trajectories; given an i.i.d. sequence of trajectories $(x^{(k)}_t, a^{(k)}_t, r^{(k)}_t)_{t \geq 0}$, we write $\hat{\alpharetraceContract}^{(k)}(\alpha(\phi))$ for the corresponding estimates of $\alpharetraceContract(\alpha(\phi))$. 
If the target contraction rate is $\Gamma$, we can adjust an initial parameter $\phi_0 \in \mathbb{R}$ using these estimates according to the Robbins-Monro rule
\begin{align}\label{eq:SAupdate}
    \phi_{k+1} = \phi_k - \varepsilon_k\left( \hat{\alpharetraceContract}^{(k)}(\alpha(\phi_k)) - \Gamma \right)  \qquad \forall k \geq 0 \, ,
\end{align}
for some sequence of stepsizes $(\varepsilon_k)_{k=0}^\infty$. The following result gives a theoretical guarantee for the correctness of this procedure.

\begin{restatable}{proposition}{propSA}\label{prop:sa}
    Let $(x^{(k)}_t,a^{(k)}_t, r^{(k)}_t)_{t=0}^\infty$ be an i.i.d. sequence of trajectories following $\mu$, with initial state-action distribution given by $\nu$. Let $\Gamma$ be a target contraction rate such that $\alpharetraceContract_\nu(1) \geq \Gamma$. Let the stepsizes $(\varepsilon_k)_{k=0}^\infty$ satisfy the usual Robbins-Monro conditions $\sum_{k=0}^\infty \varepsilon_k = \infty$, $\sum_{k=0}^\infty \varepsilon_k^2 < \infty$. Then for any initial value $\phi_0$ following the updates in \eqref{eq:SAupdate}, we have $\phi_k \rightarrow \phi^*$ in probability, where $\phi^* \in \mathbb{R}$ is the unique value such that $\alpharetraceContract_\nu(\alpha(\phi^*)) = \Gamma$.
\end{restatable}

\ctrace\ thus consists of interleaving \alpharetrace\ evaluation updates with $\alpha$ parameter updates as in \eqref{eq:SAupdate}.

\textbf{Convergence analysis.} It is possible to further develop the theory in Proposition~\ref{prop:sa} to prove convergence of C-trace as a whole, using techniques going back to those of \citet{bertsekas1996neuro} for convergence of TD($\lambda$), and more recently used by \citet{munos2016safe} to prove convergence of a control version of Retrace, as the following result shows.

\begin{restatable}{theorem}{ThmCtraceConvergence}
    Assume the same conditions as Proposition 3.4, and additionally that: (i) trajectory lengths have finite second moment; (ii) immediate rewards are bounded.
    Let $(\phi_k)_{k=0}^\infty$ be defined as in Equation~\eqref{eq:SAupdate} and $(Q_k)_{k=0}^\infty$ be a sequence of Q-functions, with $Q_{k+1}$ obtained from applying Retrace updates targeting $\alpha(\phi_k)\pi + (1-\alpha(\phi_k))\mu$ to $Q_k$ with trajectory $k+1$, using stepsize $\varepsilon_k$. 
    Then we have $\alpha(\phi_k) \rightarrow \alpha(\phi^*) =: \alpha^*$ and $Q_k \rightarrow Q^{\alpha^* \pi + (1-\alpha^*)\mu}$ almost surely.
\end{restatable}

\textbf{Truncated trajectory corrections.} The method described above for adaptation of $\alpha$ is impractical in scenarios where episodes are particularly long, when the MDP is non-episodic, and when only partial segments of trajectories can be processed at once. Since such cases often arise in practice, this motivates modifications to the update of \eqref{eq:SAupdate}. Here, we describe one such modification which will be crucial to the deployment of \ctrace\ in large-scale agents in Section \ref{sec:large-experiments}. Given a \emph{truncated trajectory} $(x_t, a_t, r_t)_{t=0}^N$, Retrace necessarily must cut traces after at most $N$ time steps, and so can achieve a contraction rate of $\gamma^N$ at the very lowest. We thus adjust the target contraction rate accordingly, and arrive at the following update:
\begin{align}
    \phi_{k+1} = \phi_k - 
    \varepsilon_k \left( \hat{\alpharetraceContract}^{(k)}(\alpha(\phi_k)) - \max(\Gamma, \gamma^N) \right) \, .
\end{align}

\section{Experiments}\label{sec:experiments}\label{sec:large-experiments}

\begin{figure*}
    \centering
    \null\hfill
    \includegraphics[keepaspectratio,height=.17\textheight]{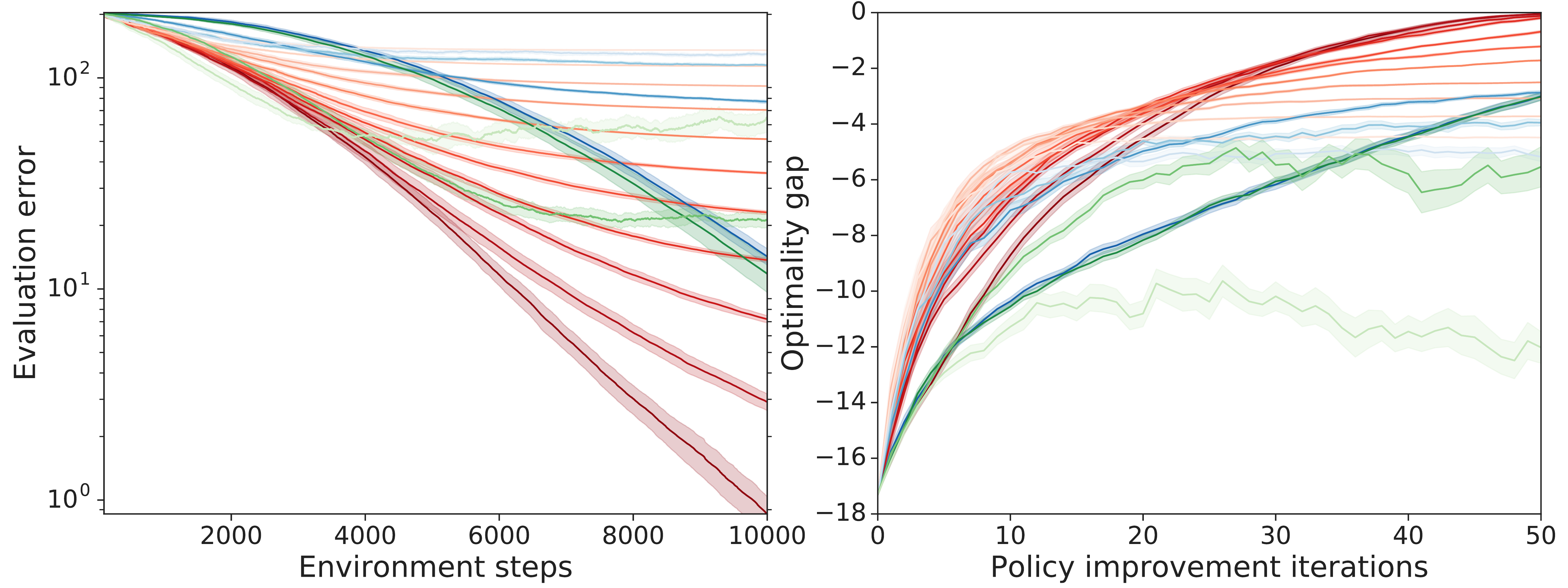}
    \includegraphics[keepaspectratio,width=.25\textwidth]{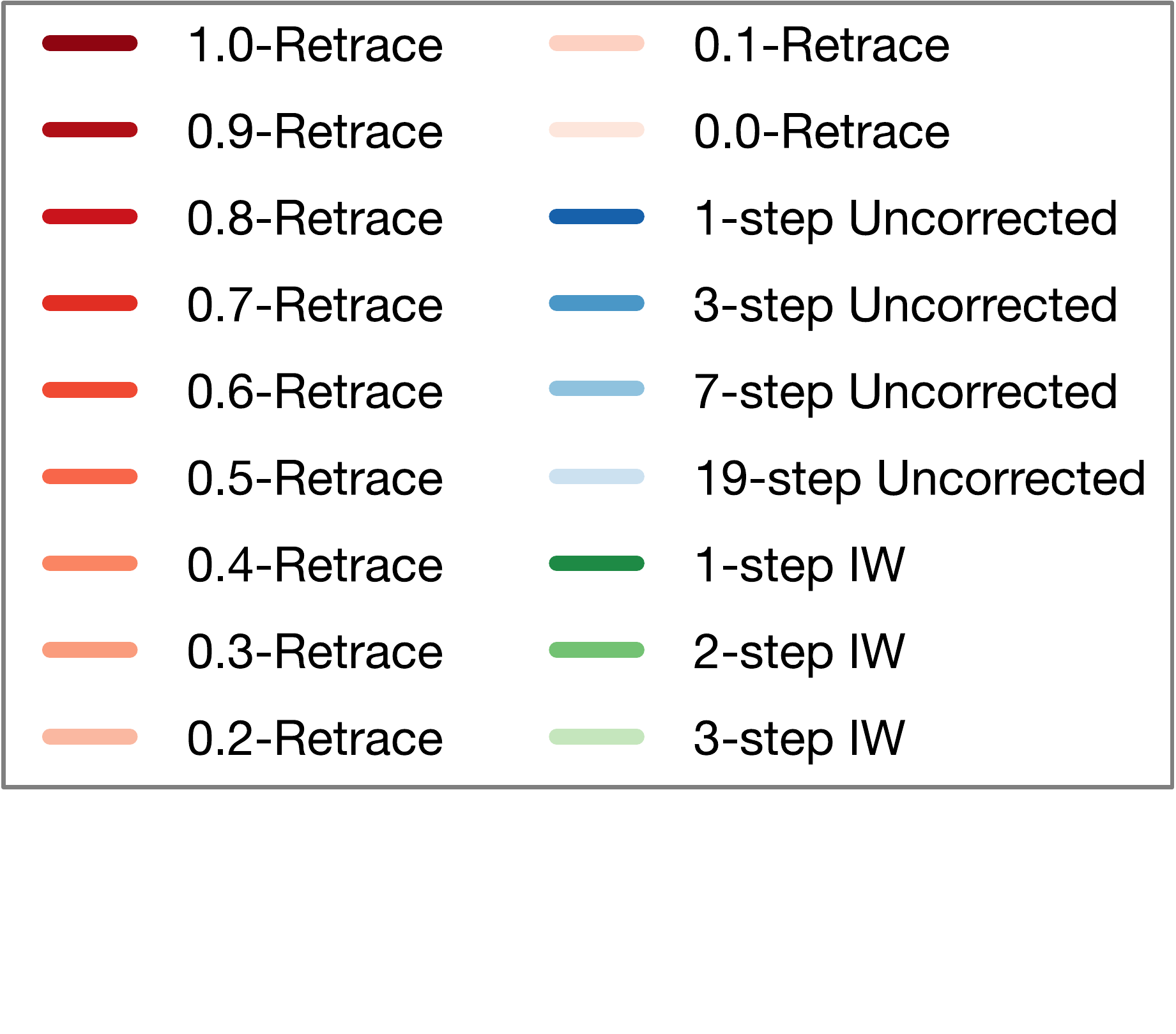}
    \hfill\null
    \caption{Left: Performance of a variety of off-policy evaluation methods on a small MDP; for further details, see text of Section \ref{sec:evalexperiments}.
    Right: Performance of a variety of modified policy iteration methods on a small MDP; for further details, see text of Section \ref{sec:controlexperiments}.}
    \label{fig:eval_control_plot}
\end{figure*}

Having explored the types of trade-offs \alpharetrace\ makes relative to existing off-policy algorithms, we now investigate the performance of these methods in the downstream tasks of evaluation and control described in Section \ref{sec:downstreamtasks}.

\textbf{Evaluation. }\label{sec:evalexperiments}
In the left sub-plot of Figure \ref{fig:eval_control_plot}, we compare the performance of \alpharetrace, $n$-step uncorrected updates, and $n$-step importance-weighted updates, for various values of the parameters concerned, at an off-policy evaluation task. In this particular task, the environment is given by a chain MDP (see Appendix Section \ref{sec:envs}), the target policy is optimal, and the behaviour is the uniform policy. We plot Q-function $L^2$ error against number of environment steps; see full details in Appendix Section \ref{sec:mainplotsmoredetails}. Standard error is indicated by the shaded regions.

The best performing methods vary as a function of the number of environment steps experienced. For low numbers of environment steps, the best performing methods are $n$-step uncorrected updates for large $n$, and \alpharetrace\ for $\alpha$ close to $0$. Intuitively, in this regime, a good contraction rate outweighs fixed-point bias. As the number of environment steps increases, the fixed-point bias kicks in, and the optimally-performing $\alpha$ gradually increases from close to $0$ to close to $1$. Note that typically the high variance of the importance-weighted updates preclude them from attaining any reasonable level of evaluation error.

\textbf{Control. }\label{sec:controlexperiments}
In the right sub-plot of Figure \ref{fig:eval_control_plot}, we compare the performance of a variety of modified policy iteration methods, each using a different off-policy evaluation method. We use the same MDP as in the evaluation example above, and plot the sub-optimality of the learned policy (measured as difference between expected return under a uniformly random initial state for optimal and learned policies) against the number of policy improvement steps performed. In this experiment, the behaviour policy is fixed as uniform throughout.
As with evaluation, we see that initial improvements in policy are strongest with highly-contractive methods incorporating some fixed-point bias, with less-biased approaches catching up (and ultimately surpassing) for greater amounts of environment interaction.

\subsection{\ctrace-R2D2}

To test the performance of our methods at scale, we adapted R2D2 \citep{kapturowski2018recurrent} to replace the original $n$-step uncorrected update with Retrace and \ctrace. For \ctrace\ we targeted the contraction rate given by an $n$-step uncorrected update, using a discount rate of $\gamma=0.997$ and $n=10$. Based on the Pareto efficiency of \alpharetrace\ relative to $n$-step uncorrected returns
exhibited empirically in small-scale MDPs, we conjectured that this should lead to improved performance. The agent was trained on the Atari-57 suite of environments \citep{bellemare13arcade} with the same experimental setup as in \citep{kapturowski2018recurrent}, a description of which we include in Appendix Section \ref{sec:r2d2details}. High-level results are displayed in Figure \ref{fig:r2d2results}, plotting mean human-normalised performance, median human-normalised performance, and mean human-gap (across the 57 games) against wall-clock training time; detailed results are given in Appendix Section \ref{sec:r2d2detailedresults}. We also provide empirical verification that C-trace-R2D2 successfully attains its target contraction rate in practice in Appendix Section~\ref{sec:contraction-verification}.

\ctrace-R2D2 attains comparable or superior performance relative to R2D2 and Retrace-R2D2 in all three performance measures. Thus, not only does C-trace-R2D2 match state-of-the-art performance for distributed value-based agents on Atari, it also achieves the earlier stated goal of bridging the gap between the performance of uncorrected returns and more principled off-policy algorithms in deep reinforcement learning.

\subsection{\ctrace-DQN}

\begin{figure*}
	\centering
	\includegraphics[keepaspectratio, width=.85\textwidth]{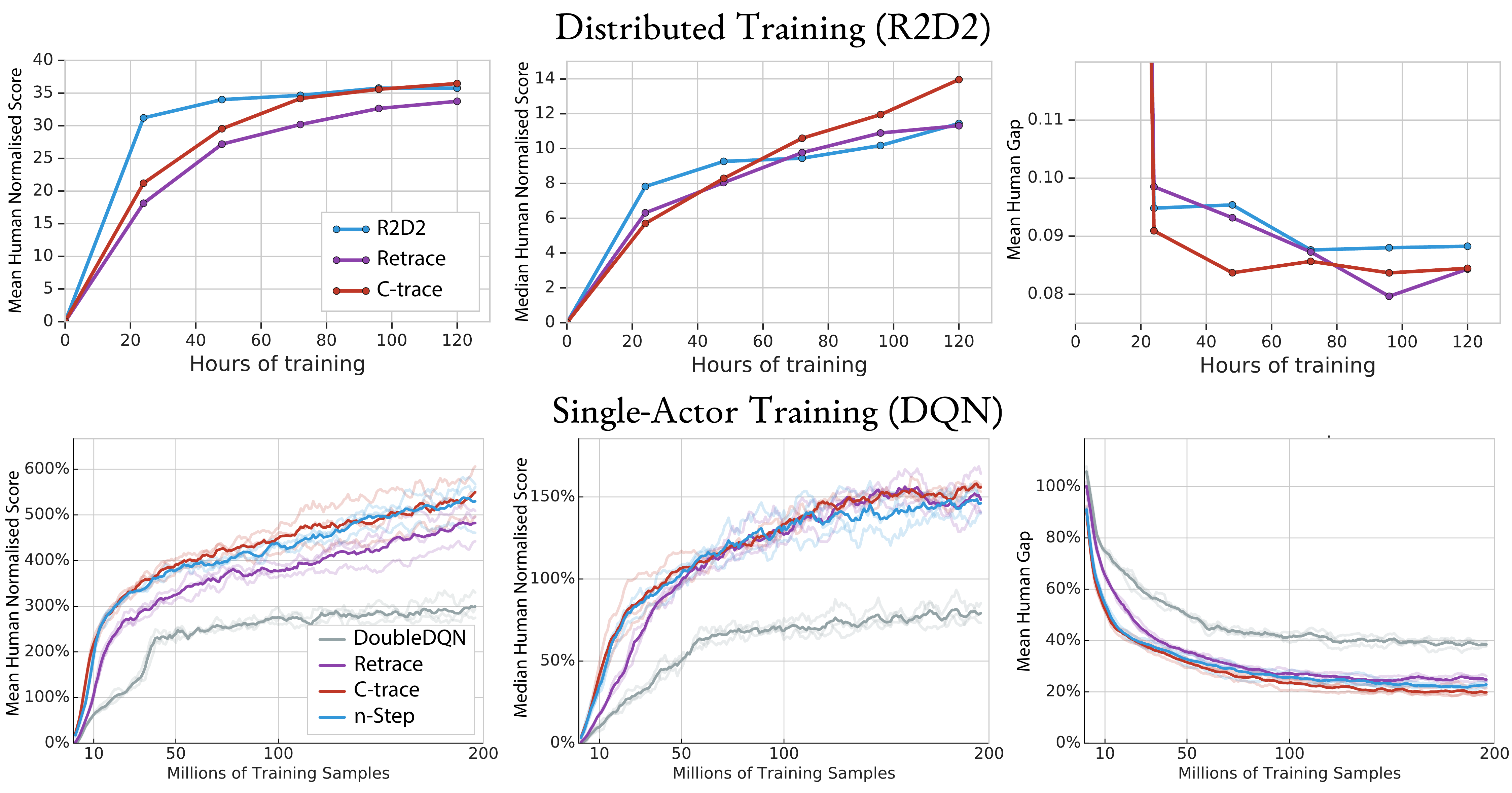}
\caption{High-level performance of variants of R2D2 (top row) and DQN (bottom row) on the Atari suite of environments. R2D2-based methods are averages of two seeds. DQN-based methods are averages of three seeds. (\textbf{Left}) Mean human-normalised score, (\textbf{Centre}) median human-normalised score, and (\textbf{Right}) human gap.}
	\label{fig:r2d2results}
\end{figure*}

To illustrate the flexibility of \ctrace\ as an off-policy learning algorithm, we also demonstrate its performance within a DQN architecture \citep{mnih2015human}. We use Double DQN \citep{van2016deep} as a baseline, and modify the one-step Q-learning rule to use $n$-step uncorrected returns, Retrace, and \ctrace. As for the R2D2 experiment, we set the C-trace contraction target using $n=10$, demonstrating the robustness of this C-trace hyperparameter across different architectures. Further, we found the behaviour of C-trace to be generally robust to the choice of $n$; see Appendix Section~\ref{sec:dqndetailedresults}. Full experimental specifications are given in Appendix Section \ref{sec:dqndetails}, with detailed results in Appendix Section \ref{sec:dqndetailedresults}; a high-level summary is displayed in Figure \ref{fig:r2d2results}. All sequence-based methods significantly outperform Double DQN, as we would expect. We notice that the performance gap between $n$-step and Retrace is not as large here as for R2D2. A possible explanation for this is that the data distribution used by DQN is typically ``more off-policy'' than R2D2, as the latter uses a distributed set of actors to increase data throughput.
As with the R2D2 experiments we see that \ctrace-DQN achieves similar learning speed as the targeted $n$-step update, but with improved final performance. One interpretation of these results is that the improved contraction rate of \ctrace\ allows it to learn significantly faster than Retrace, while the better fixed-point error allows it to find a better long-term solution than $n$-step uncorrected.

\section{Related work}\label{sec:relatedwork}

A central observation of this work is that the fixed-point bias can be explicitly traded-off to improve contraction rates. To our understanding, this is the first work to directly study this possibility, and further to draw attention to three fundamental quantities to be traded-off in off-policy learning. However, investigating trade-offs in off-policy RL, and in particular parametrising methods to allow a spectrum of algorithms is a long-standing research topic \citep{sutton2018reinforcement}. 
The most closely related methods come from a line of work that consider the bias-variance trade-off due to bootstrapping. In our framework, we understand this as a trade-off between variance and contraction rate, \textit{but without modifying the fixed-point}. The recently introduced $Q(\sigma)$ algorithm uses the $\sigma$ hyperparameter to mix between importance-weighted $n$-step SARSA and TreeBackup \citep{de2018multi}. In another recent related approach, \citet{shi2019tbq} uses $\sigma$ to mix between TreeBackup$(\lambda)$ and $Q(\lambda)$, although neither of these approaches adaptively set $\sigma$ based on observed data.
We have developed an adaptive method for adjusting $\alpha$ to achieve a desired trace length, and believe an interesting direction for future work would be to develop the adaptive methods described in this paper for use in other families of off-policy learning algorithms.

Conservatively updating policies within control algorithms is a well-established practice; \citet{kakade2002approximately} consider a trust-region method for policy improvement, motivated by inexact policy evaluation due to function approximation. In contrast, in this work we consider regularised policy improvement as a means of improving evaluation of future policies, even in the absence of function approximation. More recently, this approach also led to several advances in policy gradient methods \citep{schulman2015trust,schulman2017proximal} based on trust regions.
Although not the focus of this work, there has been also been much progress on correcting state-visitation distributions \citep{sutton2016emphatic,thomas2016data,hallak2017consistent,liu2018breaking}, another form of off-policy correction important in function approximation, as illustrated in the classic counterexample of \citet{baird1995residual}.

\section{Conclusion}\label{sec:conclusion}

We have highlighted the fundamental role of variance, fixed-point bias, and contraction rate in off-policy learning, and described how existing methods trade off these quantities. With this perspective, we developed novel off-policy learning methods, \alpharetrace\ and \ctrace, and incorporated the latter into several deep RL agents, leading to strong empirical performance.

Interesting questions for future work include applying the adaptive ideas underlying C-trace to other families of off-policy algorithms, investigating whether there exist new off-policy learning algorithms in unexplored areas of the space of trade-offs, and developing a deeper understanding of the relationship between these fundamental properties of off-policy learning algorithms and downstream performance on large-scale control tasks.


\section*{Acknowledgements}

Thanks in particular to Hado van Hasselt for detailed comments and suggestions on an earlier version of this paper, and thanks to
Bernardo Avila Pires,
Diana Borsa, 
Steven Kapturowski,
Bilal Piot, 
Tom Schaul,
and Yunhao Tang
for interesting conversations during the course of this work. Thanks also to the anonymous reviewers for helpful comments during the review process.

\bibliographystyle{plainnat}
\bibliography{main}

\clearpage
\onecolumn
\appendix
\section*{\centering APPENDICES: Adaptive Trade-Offs in Off-Policy Learning}

\section{Proofs}

\propDecomposition*

\begin{proof}
    
    Note that by the triangle inequality:
    \begin{align*}
        \mathbb{E}\left\lbrack \| \hat{T}Q - Q^\pi \|_\infty \right\rbrack \leq  \mathbb{E}\left\lbrack \| \hat{T}Q - \widetilde{T}Q \|_\infty \right\rbrack + \| \widetilde{T}Q - \widetilde{Q} \|_\infty + \| \widetilde{Q} - Q^\pi \|_\infty   \, .
    \end{align*}
    Now, observing $\|\widetilde{T}Q - \widetilde{Q}\|_\infty = \|\widetilde{T}Q - \widetilde{T}\widetilde{Q}\|_\infty \leq \Gamma \|Q - \widetilde{Q}\|_\infty$ yields the second term on the right-hand side of the stated bound. Using the inequality $\|\cdot\|_\infty \leq \|\cdot\|_2$ and Jensen's inequality yields the remaining terms.
\end{proof}

\thmLowerBound*

\begin{proof}
    The high-level approach to the proof is to exhibit two MDPs $M_0, M_1 \in \mathcal{M}$ which with high probability under the data used by $\hat{T}$, cannot be distinguished. This yields a high probability lower bound on the evaluation error that the operator $\hat{T}$ achieves on the two MDPs. This in turn implies that the mean-squared error quantity of Proposition \ref{prop:decomposition} cannot be uniformly low across $M_0$ and $M_1$, and this yields a lower bound for the quantity on the right-hand side of the bound appearing in Proposition \ref{prop:decomposition}, as required.
    
    Using the notation introduced in the statement of the theorem, for a given $\delta \in (0,1)$, let $(x^*, a^*) \in \statespace\times\actionspace$, $\Omega^* \subseteq \statespace$ be quantities achieving the maximum in the definition of $D(Z, \pi, \delta)$. Thus, with probability at least $1-\delta$, none of the state-action pairs used by the algorithm $\hat{T}$ are contained in $\Omega^*$. 
        
    Now define two MDPs $M_0$, $M_1$ with common state space $\statespace$, action space $\actionspace$, transition kernel $P$, and discount factor $\gamma$, with reward functions $r_0, r_1 : \statespace\times\actionspace \rightarrow \mathbb{R}$ defined by
    \begin{align*}r_i(x, a) =
        \begin{cases}
            0 &  \text{ if } x \not\in \Omega^* \\
            (-1)^i R_\text{max} &  \text{ if } x \in \Omega^* \, .
        \end{cases}
    \end{align*}
    Now, the Q-functions associated with these two MDPs can be calculated by $(I - \gamma P^\pi)^{-1} r_0$ and $(I - \gamma P^\pi)^{-1} r_1$, and so we can read off their difference in the $(x^*, a^*)$ coordinate as
    \begin{align*}
        \sum_{(x, a) \in \Omega^*} \frac{1}{1-\gamma} d_{(x^*, a^*), \pi}(x, a)2 R_\text{max} = \frac{2D(Z, \pi, \delta) R_\text{max}}{1-\gamma} \, .
    \end{align*}
    Thus, with probability $1-\delta$, the algorithm must make an error of at least $D(Z, \pi, \delta)R_\text{max}/(1-\gamma)$ on one of the MDPs $M_0$ and $M_1$, as measured the $L^\infty$ norm. This implies that the expected $L^\infty$ error appearing on the left-hand side of the bound in Proposition \ref{prop:decomposition} is at least $(1-\delta)D(Z, \pi, \delta)R_\text{max}/(1-\gamma)$ for one of the MDPs $M_0$ and $M_1$. Thus, the statement of the theorem follows by taking a supremum over $\delta \in (0,1)$.
\end{proof}

\propContractionRate*

\begin{proof}
    The \alpharetrace\ operator for evaluation of $\pi$ given behaviour $\mu$ corresponds to the standard Retrace operator for evaluation of $\pi^\alpha = \alpha\pi + (1-\alpha)\mu$ given behaviour $\mu$. Thus, from the analysis of \citet{munos2016safe}, the contraction rate of the \alpharetrace\ operator specific to a particular state-action pair $(x, a) \in \statespace\times\actionspace$ may be immediately written down as
    \begin{align*}
        & 1 - (1-\gamma) \mathbb{E}_\mu\left\lbrack 
            \sum_{t\geq0} \gamma^t \prod_{s=1}^t \min\left(1, \frac{\alpha\pi(A_t|X_t) + (1-\alpha)\mu(A_t|X_t)}{\mu(A_t|X_t)} \right)
        \middle| X_0 = x, A_0 = a \right\rbrack \\
        = & 1 - (1-\gamma) \mathbb{E}_\mu\left\lbrack 
            \sum_{t\geq0} \gamma^t \prod_{s=1}^t \left( (1-\alpha) + \alpha \min\left(1, \frac{\pi(A_t|X_t)}{\mu(A_t|X_t)} \right)\right)
        \middle| X_0 = x, A_0 = a \right\rbrack \, .
    \end{align*}
    To see that this is a continuous function of $\alpha \in [0,1]$, we note that the integrand of the expectation above is clearly a continuous function of $\alpha$, and is uniformly dominated by the constant function equal to $(1-\gamma)^{-1}$. By the dominated convergence theorem, continuity of the above expression follows. Since the integrand is non-negative and bounded above by $(1-\gamma)^{-1}$, the contraction rate must lie in the interval $[0, \gamma]$ for all $\alpha \in [0,1]$. For monotonicity, we show that each term
    \begin{align}\label{eq:individterm}
        \mathbb{E}_\mu\left\lbrack 
            \prod_{s=1}^t \left( (1-\alpha) + \alpha \min\left(1, \frac{\pi(A_t|X_t)}{\mu(A_t|X_t)} \right)\right)
            \middle| X_0 = x, A_0 = a
        \right\rbrack
    \end{align}
    is monotonic decreasing in $\alpha$, meaning that the contraction rate is monotonic increasing in $\alpha$. To this end, observe that the integrand of the expectation above almost-surely takes the form $\prod_{s=1}^t (1 - z_s \alpha)$ for some coefficients $z_s \in [0,1]$. The derivative with respect to $\alpha$ of this expression is $\sum_{s=1}^t -z_s \prod_{s^\prime\not= s}^t (1 - z_{s^\prime} \alpha)$, which is non-positive for $\alpha \in [0,1]$. It is again straightforward to apply the dominated convergence theorem to this derivative to obtain that the derivative of Expression \eqref{eq:individterm} is non-positive for all $\alpha \in [0,1]$, and we thus obtain monotonicity as required. Finally, for strict monotonicity, note that if $\pi$ and $\mu$ are \emph{not} distinguishable under $\mu$, then all truncated importance weights in the expressions above are equal to $1$ almost-surely under the distribution over states visited when following $\mu$. Hence, the contraction rate is in fact constant as a function of $\alpha$, and we therefore do not have strict monotonicity. On the other hand, if $\pi$ and $\mu$ are $(x, a)$-distinguishable under $\mu$, then there exists a $t \in \mathbb{N}$ such that in the integrand of the expectation in Expression \eqref{eq:individterm}, in one of the terms constituting the product, the coefficient of $\alpha$ is less than $0$ with positive probability. Thus, the integrand is strictly monotonic with positive probability, and hence Expression \eqref{eq:individterm} itself is strictly monotonic, proving strict monotonicity of the contraction rate, as required.
\end{proof}

\thmbettercontraction*

\begin{proof}
    That the greedy policies coincide follows as a consequence of the continuity of $Q^\nu$ with respect to the policy $\nu$ and the positivity of the minimum action gap $\Delta = \inf_{x\in\statespace, a \not= a^*(x)} (Q^\pi(x, a^*(x)) - Q^\pi(x, a))$; $\alpha$ may be selected so that e.g. $\|Q^{\pi^\alpha} - Q^\pi \|_\infty \leq \Delta/2$. The contraction result follows from the monotonicity property derived in Proposition \ref{prop:contractionrate}.
\end{proof}

\propSA*

\begin{proof}
    The proof follows from an application of standard stochastic approximation theory to the solution of the root-finding problem $\alpharetraceContract_\nu(\alpha(\phi)) = \Gamma$. Firstly, by Proposition \ref{prop:contractionrate}, the function $\phi \mapsto \alpharetraceContract_\nu(\alpha(\phi))$ is continuous and monotonic on $\mathbb{R}$. By the assumption that $\alpharetraceContract_\nu(1) \geq \Gamma$, it follows that $\phi \mapsto \alpharetraceContract_\nu(\alpha(\phi))$ is strictly monotonic, and moreover by inspecting the proof of Proposition \ref{prop:contractionrate}, has positive derivative everywhere. By the intermediate value theorem, there exists a unique value $\phi^* \in \mathbb{R}$ such that $\alpharetraceContract_\nu(\alpha(\phi^*)) = \Gamma$.
    
    Now note that for each $\phi \in \mathbb{R}$, the random variables $\hat{\alpharetraceContract}^{(k)}(\alpha(\phi))$ are i.i.d. unbiased, bounded estimators of $\alpharetraceContract_\nu(\alpha(\phi))$. 
    Thus, the scheme \eqref{eq:SAupdate} is a standard stochastic approximation scheme for the the root of a monotonic function, and the conditions of Theorem 2 of \citep{robbins1951} are satisfied, enabling us to conclude that $\phi_t \rightarrow \phi^*$ in probability, as required. 
\end{proof}

\ThmCtraceConvergence*

\begin{proof}
    Convergence in probability of $\alpha_k := \alpha(\phi_k)$ to $\alpha^*$ has been shown in Proposition 3.4; it is straightforward to upgrade this to almost-sure convergence using standard stochastic approximation theory. The intuition for the remainder of the proof is that when $\alpha_k$ is close to $\alpha^*$, the C-trace updates are close to those of standard Retrace targeting the policy $\alpha^* \pi + (1-\alpha^*)\mu$, which are known to converge under the conditions of the theorem. This is made rigorous by decomposing the update on the Q-function from the $(k+1)$\textsuperscript{th} trajectory as
    \begin{align*}
        Q_{k+1}  =  \overbrace{(\mathbf{1} - \widetilde{\varepsilon}_k) \odot Q_k  +  \widetilde{\varepsilon}_k \odot \mathcal{R}^{\alpha^*} Q_k}^{\text{Desired update}} + \overbrace{(Q_{k+1} - (\mathbf{1}-\widetilde{\varepsilon}_k) \odot Q_k - \widetilde{\varepsilon}_k \odot \mathcal{R}^{\alpha_k} Q_k)}^{\text{Martingale noise}} + \widetilde{\varepsilon}_k\odot  \overbrace{(\mathcal{R}^{\alpha_k} Q_k - \mathcal{R}^{\alpha^*} Q_k)}^{\text{Perturbation}}\, ,
    \end{align*}
    where $\mathcal{R}^\alpha$ denotes the Retrace operator targeting $\alpha\pi + (1-\alpha)\mu$, 
    and with $\widetilde{\varepsilon}_k(x, a) = \varepsilon_k \mathbb{E}[\sum_{t} \mathbbm{1}_{(x_t,a_t)=(x, a)}|(x_0,a_0)=(x,a)]$, and $\odot$ the Hadamard product and $\mathbf{1}$ the vector of 1's. It is then possible to appeal to Proposition 4.5 of \citet{bertsekas1996neuro} that $Q_k \rightarrow Q^{\alpha^*\pi + (1-\alpha^*)\mu}$ almost surely, using the assumptions of theorem.
\end{proof}

\section{Additional results}\label{sec:appendixMoreResults}

\subsection{Operators for time-inhomogeneous policies}

In this section, we provide a result which rigorously proves the connection between the $n$-step uncorrected target and the time-inhomogeneous policy mentioned in Section \ref{sec:contractionbiasvariance}.

\begin{proposition}\label{prop:fixedpoints}
    The $n$-step uncorrected update corresponding to the target
    \begin{align*}
        \sum_{s=0}^{n-1} \gamma^{s} r_{s} + \gamma^n \mathbb{E}_{A \sim \pi(\cdot|x_{n})}\left\lbrack \hat{Q}(x_{n}, A)\right\rbrack \, ,
    \end{align*}
    with the trajectory generated under $\mu$, is a stochastic approximation to the operator $(T^{\mu})^{n-1}T^\pi$, with fixed point given by the $Q$-function for the time-inhomogeneous policy which follows $\pi$ at timesteps $t$ satisfying $t \equiv n-1 \text{ mod } n$, and $\mu$ otherwise.
\end{proposition}

\begin{proof}
    We begin by taking the expectation of the update target conditional on the initial state-action pair, and showing that it is equal to $((T^{\mu})^{n-1}T^\pi \hat{Q})(x_0, a_0)$. We proceed by induction. In the case $n=1$, the expectation of the update is given by
    \begin{align*}
        & \mathbb{E}_\mu\left\lbrack R(X_0, A_0) + \gamma \mathbb{E}_{A \sim \pi(\cdot|X_{1})}\left\lbrack \hat{Q}(X_{1}, A)\right\rbrack \middle|X_0 = x_0, A_0 = a_0 \right\rbrack \\
        = & r(x_0, a_0) + \sum_{x^\prime \in \statespace} P(x^\prime|x, a) \gamma \sum_{a^\prime \in \actionspace} \pi(a^\prime | x^\prime) \hat{Q}(x^\prime, a^\prime) \\
        = & (T^\pi \hat{Q})(x_0, a_0) \, ,
    \end{align*}
    as required. For the inductive step, we assume the result holds for some $n \geq 1$. Now observe that by conditioning on $(X_1, A_1)$, we have
    \begin{align*}
        & \mathbb{E}_\mu\left\lbrack \sum_{s=0}^{n} \gamma^{s} R(X_s, A_s) + \gamma^{n+1} \mathbb{E}_{A \sim \pi(\cdot|X{n})}\left\lbrack \hat{Q}(X_{n+1}, A)\right\rbrack \middle|X_0 = x_0, A_0 = a_0 \right\rbrack \\
        = & r(x_0, a_0) + \gamma \sum_{x_1 \in \statespace} P(x_1|x_0, a_0) \sum_{a_1 \in \actionspace} \mu(a_1|x_1)\times\\
        &\qquad\qquad\qquad\mathbb{E}\left\lbrack \sum_{s=1}^{n} \gamma^{s-1} R(X_s, A_s) + \gamma^{n+1} \mathbb{E}_{A \sim \pi(\cdot|X{n})}\left\lbrack \hat{Q}(X_{n+1}, A)\right\rbrack \middle| X_1 = x_1, A_1=a_1 \right\rbrack  \\
        \overset{(a)}{=} & r(x_0, a_0) + \gamma \sum_{x_1 \in \statespace} P(x_1|x_0, a_0) \sum_{a_1 \in \actionspace} \mu(a_1|x_1) ((T^\mu)^{n-1}T^\pi \hat{Q})(x_1, a_1) \\
        & = (T^\mu (T^\mu)^{n-1} T^\pi \hat{Q})(x_0, a_0) \\
        & = ((T^\mu)^{n} T^\pi \hat{Q})(x_0, a_0) \, ,
    \end{align*}
    as required, with (a) following from the induction hypothesis. Finally, for the interpretation of the fixed point of $(T^\mu)^{n-1}T^\pi$, observe that the time-inhomogeneous policy described in the statement of the proposition, which we denote $\pi\mu^{n-1}$ follows a stream of Markovian policies with period $n$, so it is possible to write down an $n$-step Bellman equation for its Q-function $Q^{\pi\mu^{n-1}}$. Doing so yields
    \begin{align*}
        Q^{\pi\mu^{n-1}}(x, a) & = \mathbb{E}_{\substack{A_{1:n-1} \sim \mu(\cdot|X_{1:n-1}) \\ A_n \sim \pi(\cdot|X_n)}}\left\lbrack \sum_{s=0}^{n-1} \gamma^s R(X_s, A_s) + \gamma^n Q^{\pi\mu^{n-1}}(X_n, A_n) \middle| X_0 = x, A_0 = a \right\rbrack \\
        & = \mathbb{E}_{\mu}\left\lbrack \sum_{s=0}^{n-1} \gamma^s R(X_s, A_s) + \gamma^n \mathbb{E}_{A_n \sim \pi(\cdot|X_n)}\left\lbrack Q^{\pi\mu^{n-1}}(X_n, A_n) \right\rbrack \middle| X_0 = x, A_0 = a \right\rbrack \, . 
    \end{align*}
    We recognise the right-hand side as the operator $(T^\mu)^{n-1}T^\pi$, and thus $Q^{\pi\mu^{n-1}}$ is the fixed point of this operator.
\end{proof}

\subsection{Further decompositions of evaluation error}

In addition to the decomposition given in Proposition \ref{prop:decomposition}, there are decompositions of evaluation error based on other norms that may be of interest. We state one such decomposition below, and also note that there is also scope to use different norms to define the fundamental traded-off quantities, such as using the $L^\infty$ norm to define an alternative notion of fixed-point bias, that lead to further decompositions.

\begin{restatable}{proposition}{propDecompositionTwo}\label{prop:decompositiontwo}
    Consider the task of evaluation of a policy $\pi$ under behaviour $\mu$, and consider an update rule $\hat{T}$ which stochastically approximates the application of an operator $\widetilde{T}$, with contraction rate $\Gamma$ and fixed point $\widetilde{Q}$, to an initial estimate $Q$. Then we have the following decomposition:
    \begin{align*}
        \mathbb{E}\left\lbrack \| \hat{T}Q - Q^\pi \|^2_2 \right\rbrack \leq 3\left\lbrack \underbrace{\mathbb{E}\left\lbrack \| \hat{T}Q - TQ \|^2_2 \right\rbrack}_{\text{Variance}} + \underbrace{\Gamma^2|\statespace||\actionspace|\| Q - \widetilde{Q} \|^2_\infty}_{\text{(Squared) contraction}} + \underbrace{\| \widetilde{Q} - Q^\pi \|^2_2}_{\text{(Squared) fixed-point bias}} \right\rbrack \, .
    \end{align*}
\end{restatable}
\begin{proof}
    The inequality is obtained in a manner analogous to that of Proposition \ref{prop:decomposition}. First, a Cauchy-Schwarz-style argument yields
    \begin{align*}
        \mathbb{E}\left\lbrack \| \hat{T}Q - Q^\pi \|^2_2 \right\rbrack \leq 3\left\lbrack \mathbb{E}\left\lbrack \| \hat{T}Q - \widetilde{T}Q \|^2_2 \right\rbrack + \| \widetilde{T}Q - \widetilde{Q} \|^2_2 + \| \widetilde{Q} - Q^\pi \|^2_2 \right\rbrack \, .
    \end{align*}
    Then, the inequality $\|\cdot\|_2 \leq |\statespace||\actionspace| \|\cdot\|_\infty$ is applied, together with the definition of $T$ as a contraction mapping under $\|\cdot\|_{\infty}$ with contraction modulus $\Gamma$, to yield the statement.
\end{proof}
\section{Experimental details}\label{sec:extra-experiments}

\subsection{Environments}\label{sec:envs}

\textbf{Dirichlet-Uniform random MDPs. } These random MDPs are specified by two parameteters: the number of states, $n_s$, and the number of actions, $n_a$. Transition distributions $P(\cdot|x, a)$ are sampled i.i.d. from a Dirchlet($1,\ldots,1$) distribution for each $\statespace\times\actionspace$. Each immediate reward distribution is given by a Dirac delta, with locations drawn i.i.d. from the Uniform($[-1,1]$) distribution.

\textbf{Garnet MDPs. } Garnet MDPs \citep{archibald1995generation,piot2014difference,bhatnagar2009natural,geist2014off} are drawn from a distribution specified by three numbers: the number of states, $n_s$, the number of actions, $n_a$, and the \emph{branching factor}, $n_b$. Each transition distribution $P(\cdot|x, a)$ is given by $n_b^{-1}\sum_{i=1}^{n_b} \delta_{z_i(x, a)}$, where $z_{1:n_b}(x, a)$ are drawn uniformly without replacement from the set of states of the MDP, independently for each state-action pair $(x, a) \in \statespace\times\actionspace$. $\lfloor n_s/10 \rfloor$ states are selected uniformly without replacement, such that any transition out of these states yields a reward of $1$, whilst all other transitions in the MDP yield a reward of $0$.

\textbf{Chain MDP. } Our chain MDP is specified by a number of states $n_s$, identified with the set $\{1,\ldots,n_s\}$. State $n_s$ is terminal. Two actions, \texttt{left} and \texttt{right}, are available at each state, which deterministically move the agent into the corresponding state (taking action \texttt{left} in state $1$ causes the agent to remain in state $1$). Every transition caused by the action \texttt{right} incurs a reward of $-1$, unless the transition is into state $n_s$, in which case a reward of 50 is received. 

\subsection{Additional details for plots appearing in the main paper}\label{sec:mainplotsmoredetails}

\textbf{Figure \ref{fig:first-tradeoff}.} We use a Dirichlet-Uniform random MDP (see Section \ref{sec:envs}) with $5$ states and $3$ actions. The target $\pi$ and behaviour $\mu$ policies were sampled independently, so that each distribution $\pi(\cdot|x)$ and $\mu(\cdot|x)$ are independent draws from the Dirichlet($1,\ldots,1$) distribution. We use a discount rate of $0.9$, and a uniform initial state distribution. The variance variable is estimated by simulating $5000$ trajectories of length $100$, from an initial Q-function estimate set to $0$.

\textbf{Figure \ref{fig:eval_control_plot}.} In both tasks, the environment is the chain described in Section \ref{sec:envs} with $n_s=20$. In both tasks, all learning algorithms use a learning rate of $0.1$, and the discount factor is set to $0.9$ throughout. In the control task, policy improvement is interleaved with $100$ steps of environment experience, which are used by the relevant evaluation algorithm. All Retrace-derived methods use $\lambda = 1$. In both evaluation and control tasks, the experiments were repeated $200$ times to estimate the standard error by bootstrapping, which is indicated in the plots by the shaded regions.
\section{Further experimental results}\label{sec:furtherexperiments}

\subsection{Further trade-off plots}\label{sec:further-tradeoff}

In this section, we give several further examples of trade-offs made by off-policy algorithms. We begin by examining the trade-offs made by TreeBackup, for which the update target (for a target policy $\pi$ given a trajectory generated according to a behaviour policy $\mu$) is stated below for completeness.
\begin{align*}
    \hat{Q}(x_0, a_0) \!+\! \sum_{s \geq 0} \gamma^s\! \prod_{u=1}^s \pi(a_u|x_u) \left(r_{s} + \gamma \mathbb{E}_{A \sim \pi(\cdot|x_{s+1})}\!\left\lbrack \hat{Q}(x_{s+1}, A) \right\rbrack\! - \hat{Q}(x_{s}, a_{s}) \right) \, .
\end{align*}
We show that mixing in a proportion $1-\alpha$ of the behaviour policy into the target in TreeBackup (which we dub \alphatreebackup) leads to fundamentally different trade-off behaviour than in \alpharetrace; see Figure \ref{fig:pareto-treebackup}. As can be seen in the plot, mixing in the behaviour policy leads to limited improvements in contraction rate relative to the trade-off achieved by \alpharetrace, whilst incurring significant fixed-point bias.

\begin{figure}
    \centering
    \includegraphics[keepaspectratio,width=1.0\textwidth]{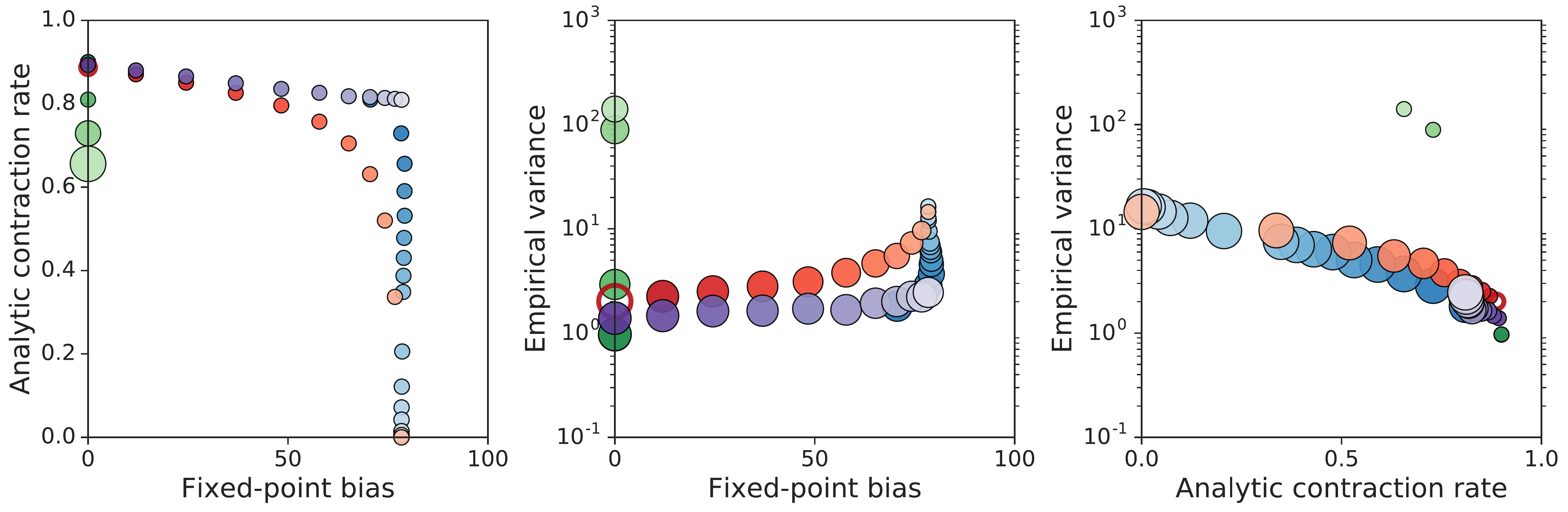}
    \caption{Trade-offs made by $n$-step uncorrected methods ($n=1$ (light blue) through to $n=50$ (dark blue), $n$-step importance-weighted methods ($n=1$ (dark green) through to $n=4$ (dark green), \alpharetrace\ ($\alpha=1$ (dark red) through to $\alpha=0$ (light red)), and \alphatreebackup\ ($\alpha=1$ (dark purple) through to $\alpha=0$ (light purple)). Results are shown for the chain environment, and evaluation of a Dirichlet($1,\ldots,1$) policy under behaviour generated by an independently sampled Dirichlet($1,\ldots,1$) policy.}
    \label{fig:pareto-treebackup}
\end{figure}

We next demonstrate the robustness of the behaviour exhibited in Figure~\ref{fig:first-tradeoff} in a variety of environments, and with a variety of target/behaviour policy pairings. As in Figure \ref{fig:first-tradeoff}, \alpharetrace\ is illustrated in red, with dark red corresponding to $\alpha=1$ through to $\alpha=0$ in light red. $n$-step uncorrected methods are illustrated in blue, ranging from $n=1$ (dark blue) through to $n=50$ (light blue). $n$-step importance-weighted methods are illustrated in green, ranging from $n=1$ (dark green) through to $n=4$ (light green). Results are given for a Dirichlet-Uniform random MDP (Figure \ref{fig:pareto-random}), a random garnet MDP (Figure~\ref{fig:pareto-garnet}), and the chain MDP described in Section \ref{sec:envs} (Figure \ref{fig:pareto-gw1d}). In all cases, we use a discount rate $\gamma=0.9$, a learning rate for each algorithm of $0.1$, and the variance variable is estimated from $5000$ i.i.d. trajectories of length $100$. All Retrace methods use $\lambda = 1$ (as presented in the main paper).

\begin{figure}
    \centering
    \begin{subfigure}{1.0\textwidth}
      \centering
      \includegraphics[keepaspectratio,width=.6\textwidth]{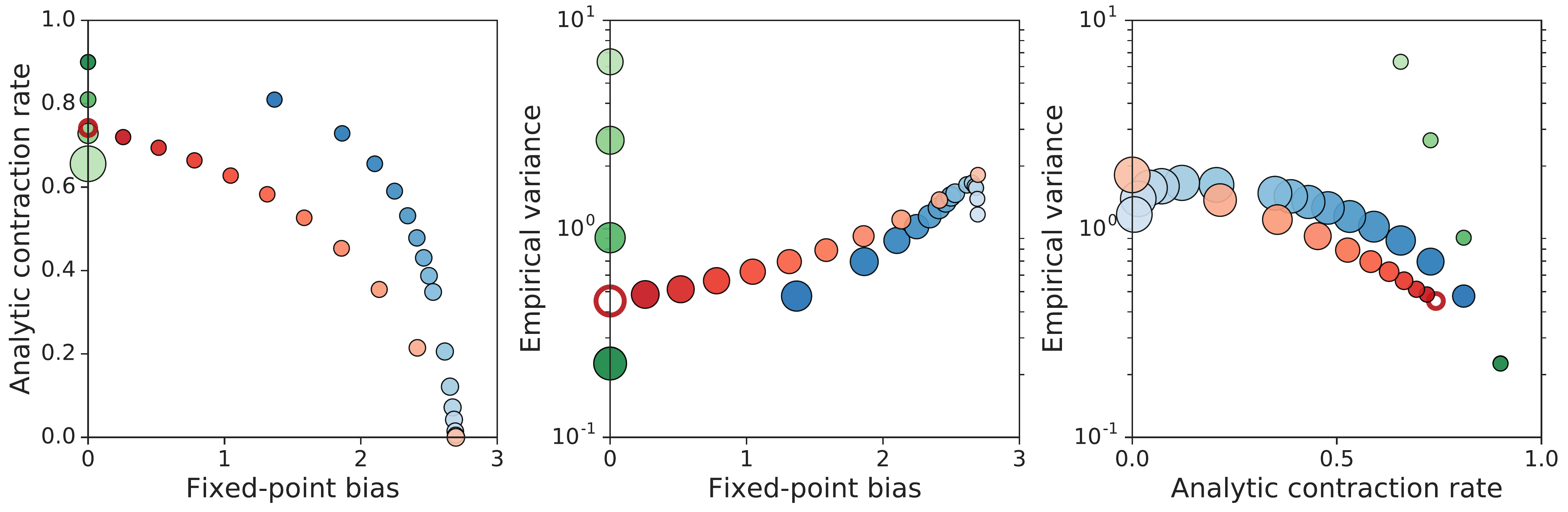}
      \vspace{-0.2cm}
      \caption{Target policy: uniform. Behaviour policy: Dirichlet($1,\ldots,1$) random.}
    \end{subfigure}%
    
    \begin{subfigure}{1.0\textwidth}
      \centering
      \includegraphics[keepaspectratio,width=.6\textwidth]{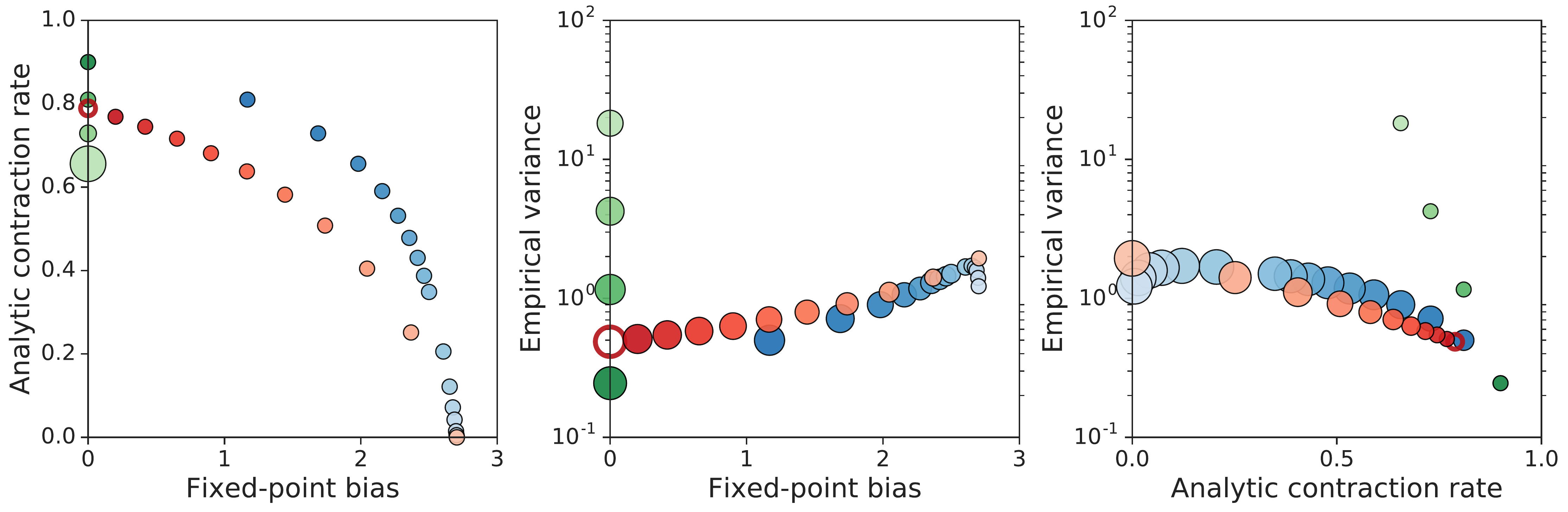}
      \vspace{-0.1cm}
      \caption{Target policy: Dirichlet($1,\ldots,1$) random. Behaviour policy: Independent Dirichlet($1,\ldots,1$) random.}
    \end{subfigure}%
    
    \begin{subfigure}{1.0\textwidth}
      \centering
      \includegraphics[keepaspectratio,width=.6\textwidth]{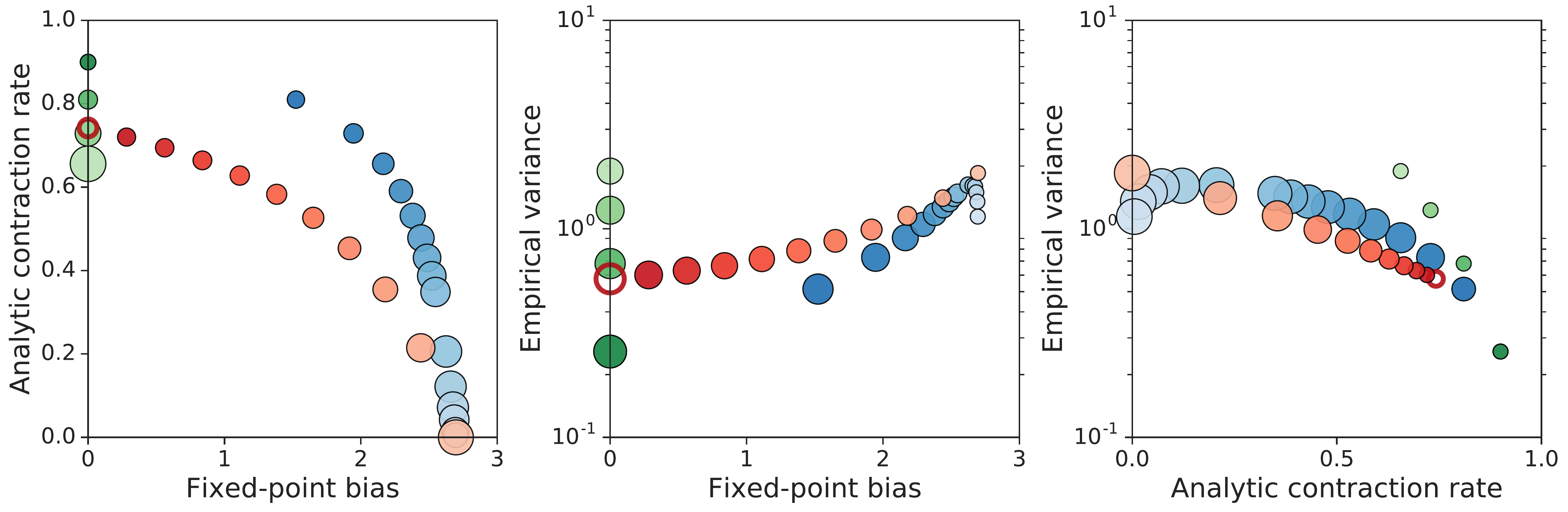}
      \vspace{-0.1cm}
      \caption{Target policy: Dirichlet($1,\ldots,1$) random. Behaviour policy: uniform.}
    \end{subfigure}%
    
    \begin{subfigure}{1.0\textwidth}
      \centering
      \includegraphics[keepaspectratio,width=.6\textwidth]{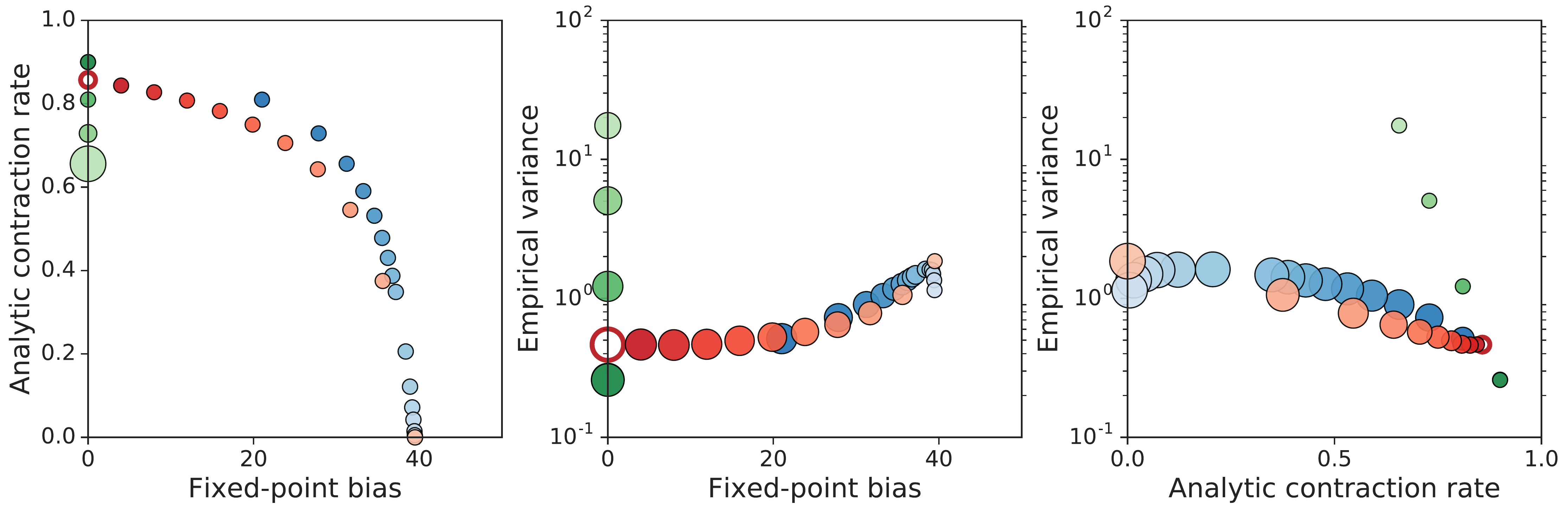}
      \vspace{-0.1cm}
      \caption{Target policy: optimal. Behaviour policy: uniform.}
    \end{subfigure}%
    
    \begin{subfigure}{1.0\textwidth}
      \centering
      \includegraphics[keepaspectratio,width=.6\textwidth]{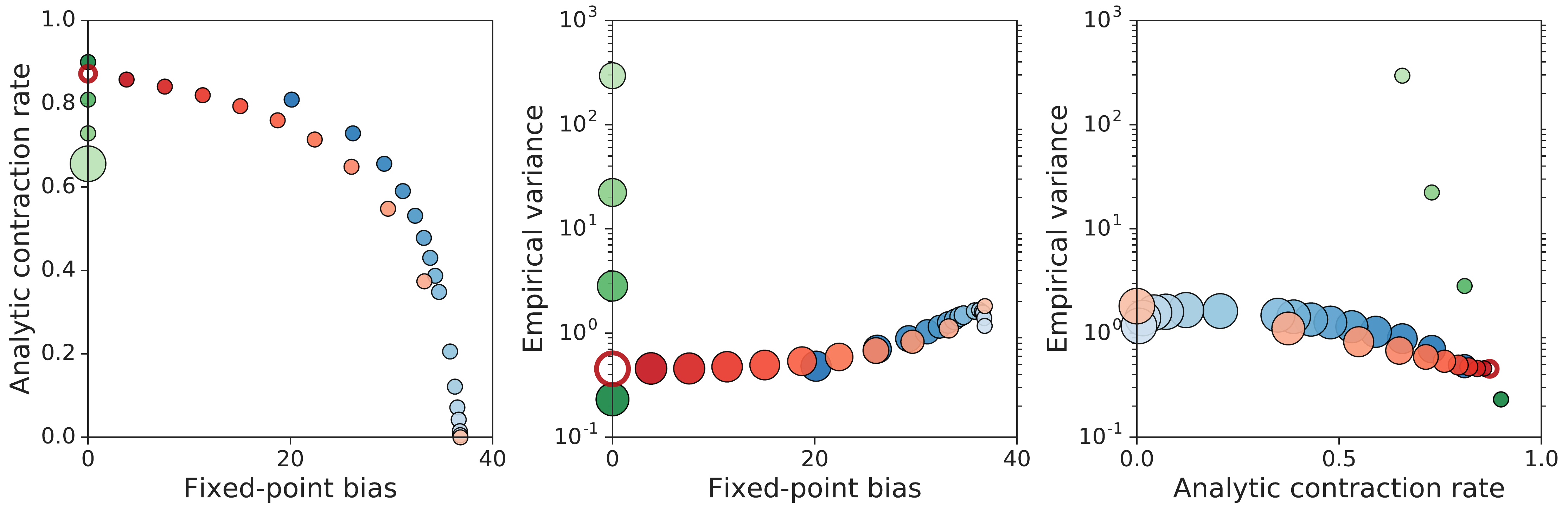}
      \vspace{-0.1cm}
      \caption{Target policy: optimal. Behaviour policy: Dirichlet($1,\ldots,1$) random.}
    \end{subfigure}%
    
    \begin{subfigure}{1.0\textwidth}
      \centering
      \includegraphics[keepaspectratio,width=.6\textwidth]{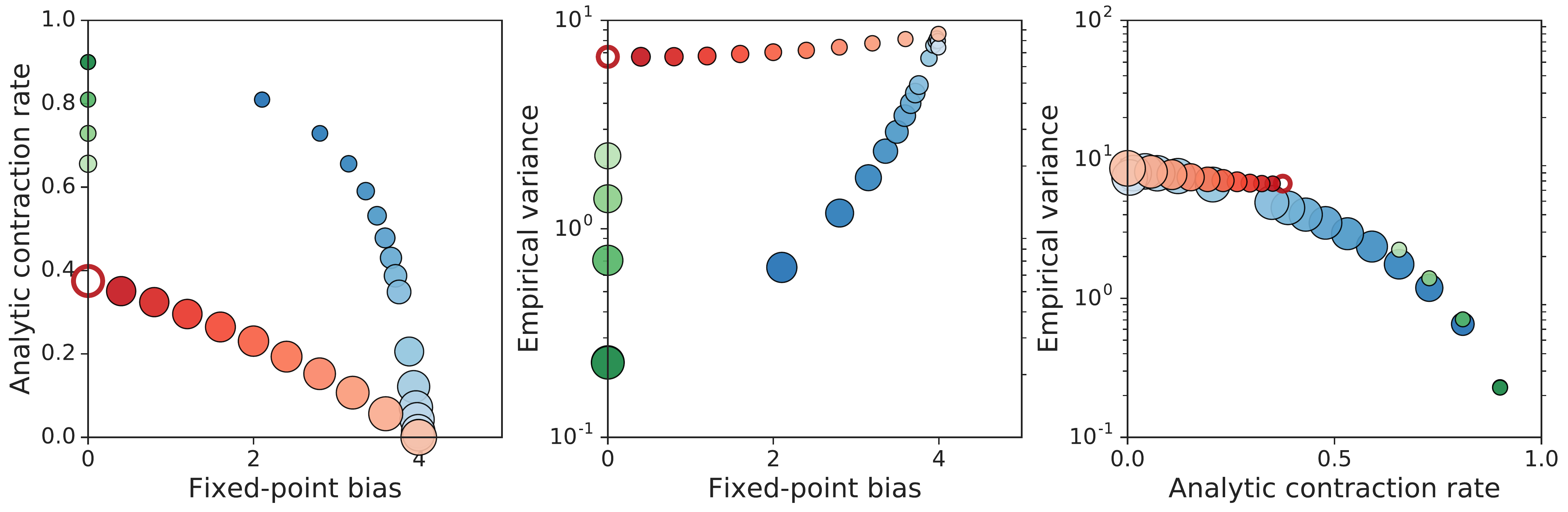}
      \vspace{-0.1cm}
      \caption{Target policy: optimal. Behaviour policy: optimal, with uniform exploration at probability $0.1$.}
    \end{subfigure}%
    
    \caption{Trade-off plots for a Dirichlet-Uniform random MDP with $20$ states and $3$ actions, with a variety of target policy/behaviour policy pairings.}
    \label{fig:pareto-random}
\end{figure}

\begin{figure}
    \centering
    \begin{subfigure}{1.0\textwidth}
      \centering
      \includegraphics[keepaspectratio,width=.6\textwidth]{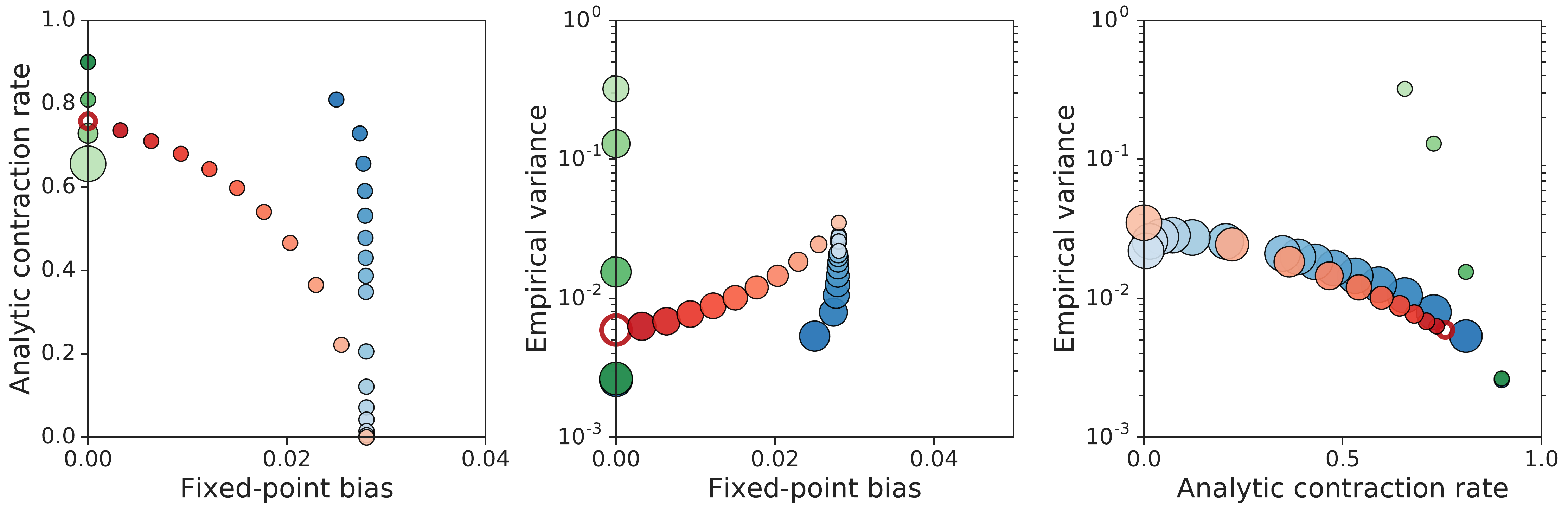}
      \vspace{-0.2cm}
      \caption{Target policy: uniform. Behaviour policy: Dirichlet($1,\ldots,1$) random.}
    \end{subfigure}%
    
    \begin{subfigure}{1.0\textwidth}
      \centering
      \includegraphics[keepaspectratio,width=.6\textwidth]{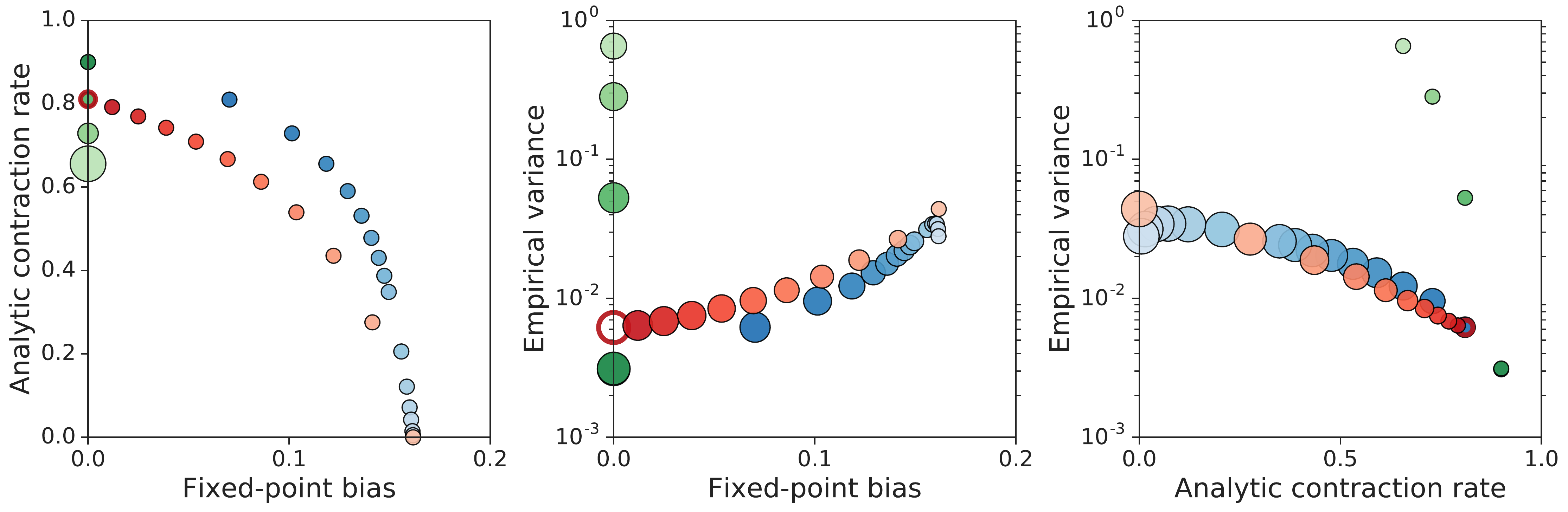}
      \vspace{-0.1cm}
      \caption{Target policy: Dirichlet($1,\ldots,1$) random. Behaviour policy: Independent Dirichlet($1,\ldots,1$) random.}
    \end{subfigure}%
    
    \begin{subfigure}{1.0\textwidth}
      \centering
      \includegraphics[keepaspectratio,width=.6\textwidth]{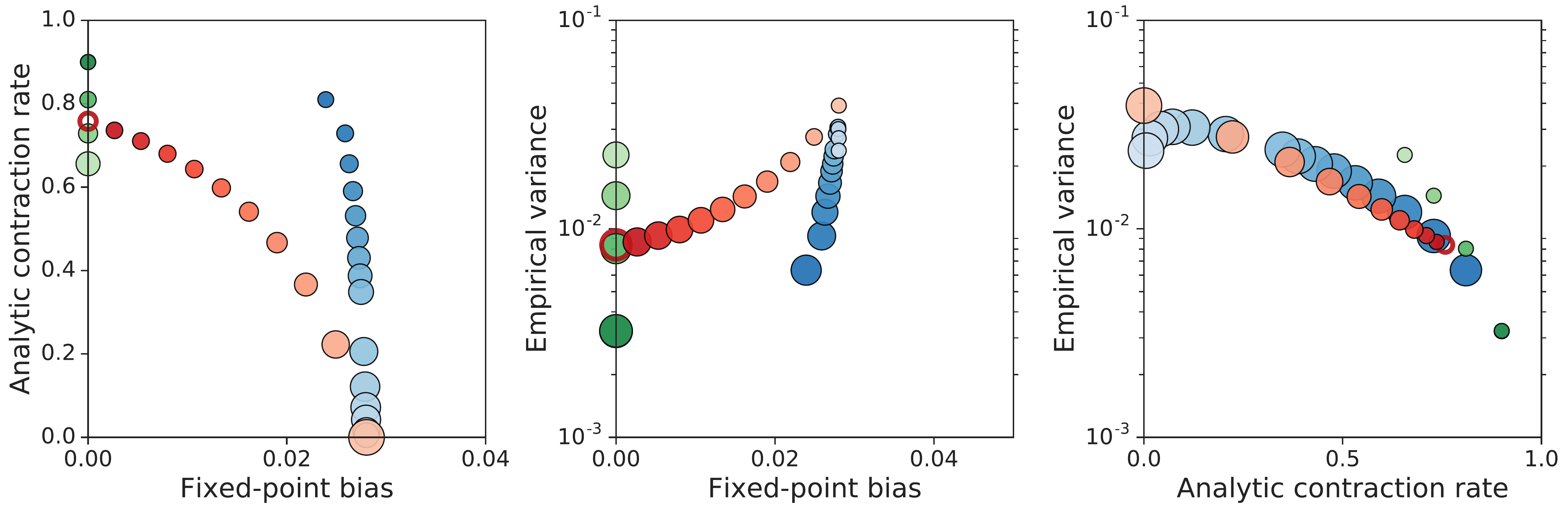}
      \vspace{-0.1cm}
      \caption{Target policy: Dirichlet($1,\ldots,1$) random. Behaviour policy: uniform.}
    \end{subfigure}%
    
    \begin{subfigure}{1.0\textwidth}
      \centering
      \includegraphics[keepaspectratio,width=.6\textwidth]{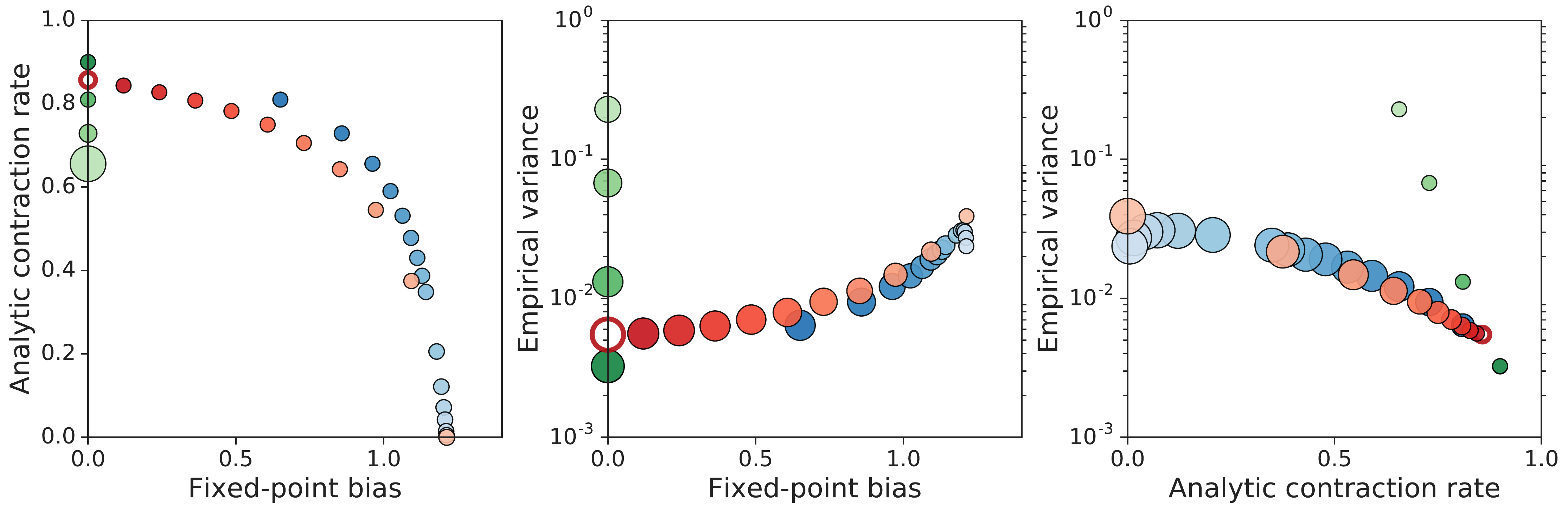}
      \vspace{-0.1cm}
      \caption{Target policy: optimal. Behaviour policy: uniform.}
    \end{subfigure}%
    
    \begin{subfigure}{1.0\textwidth}
      \centering
      \includegraphics[keepaspectratio,width=.6\textwidth]{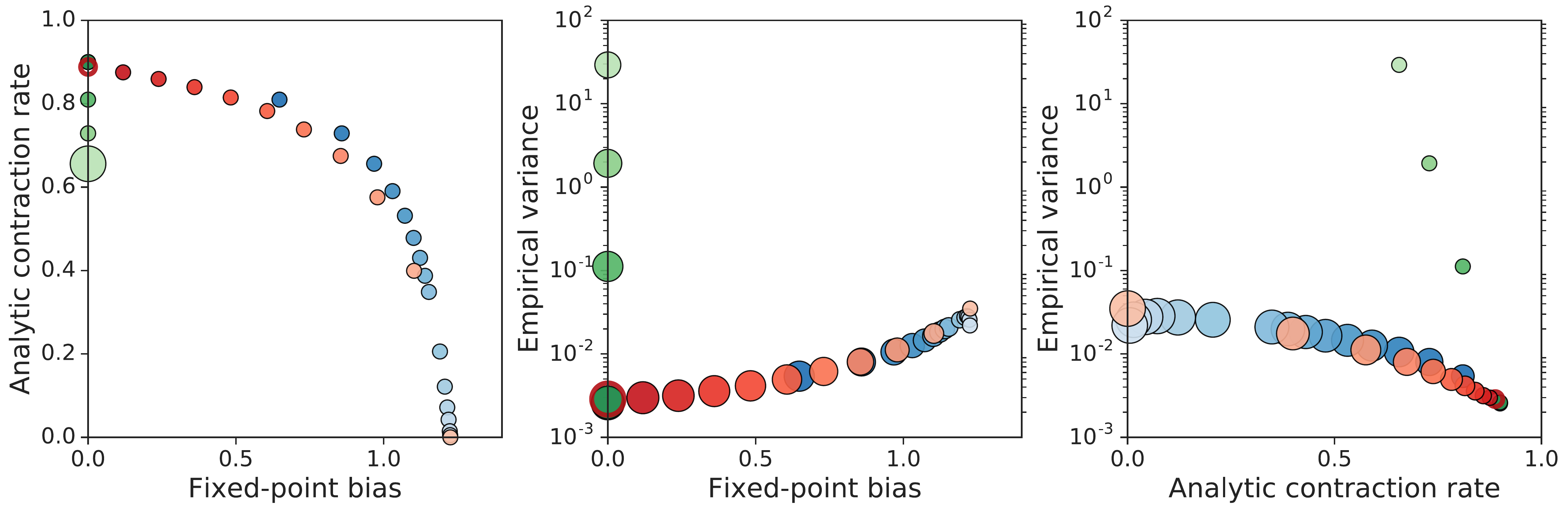}
      \vspace{-0.1cm}
      \caption{Target policy: optimal. Behaviour policy: Dirichlet($1,\ldots,1$) random.}
    \end{subfigure}%
    
    \begin{subfigure}{1.0\textwidth}
      \centering
      \includegraphics[keepaspectratio,width=.6\textwidth]{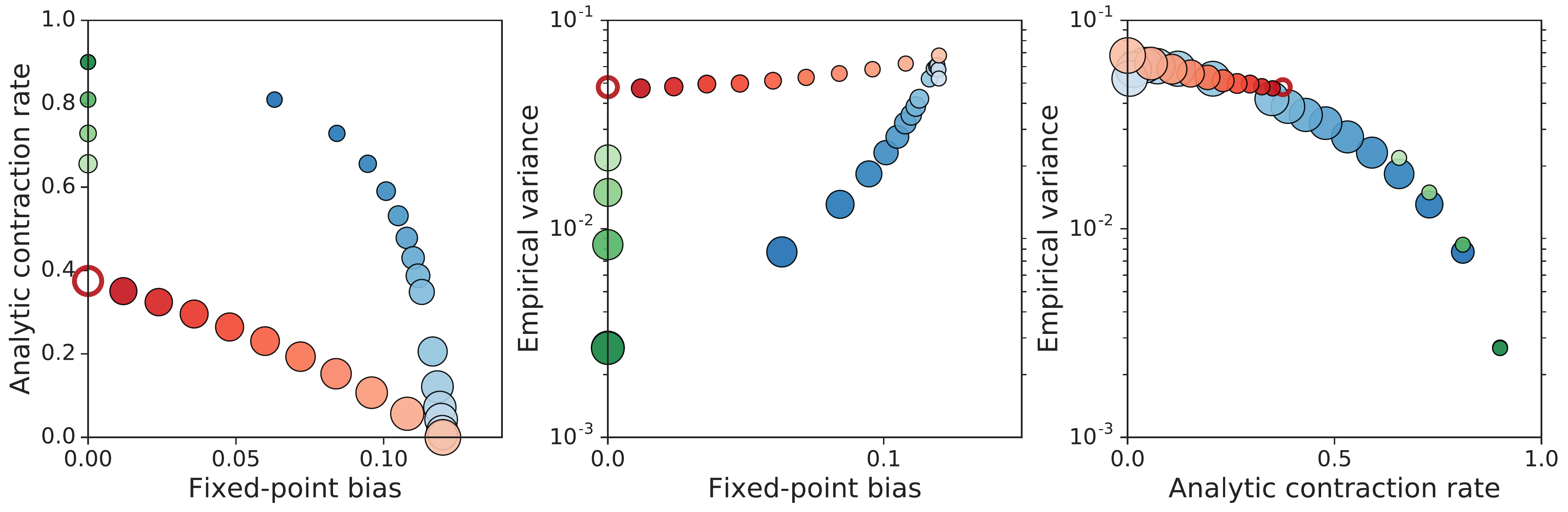}
      \vspace{-0.1cm}
      \caption{Target policy: optimal. Behaviour policy: optimal, with uniform exploration at probability $0.1$.}
    \end{subfigure}%
    
    \caption{Trade-off plots for a garnet random MDP with $20$ states and $3$ actions, with a variety of target policy/behaviour policy pairings.}
    \label{fig:pareto-garnet}
\end{figure}

\begin{figure}
    \centering
    \begin{subfigure}{1.0\textwidth}
      \centering
      \includegraphics[keepaspectratio,width=.6\textwidth]{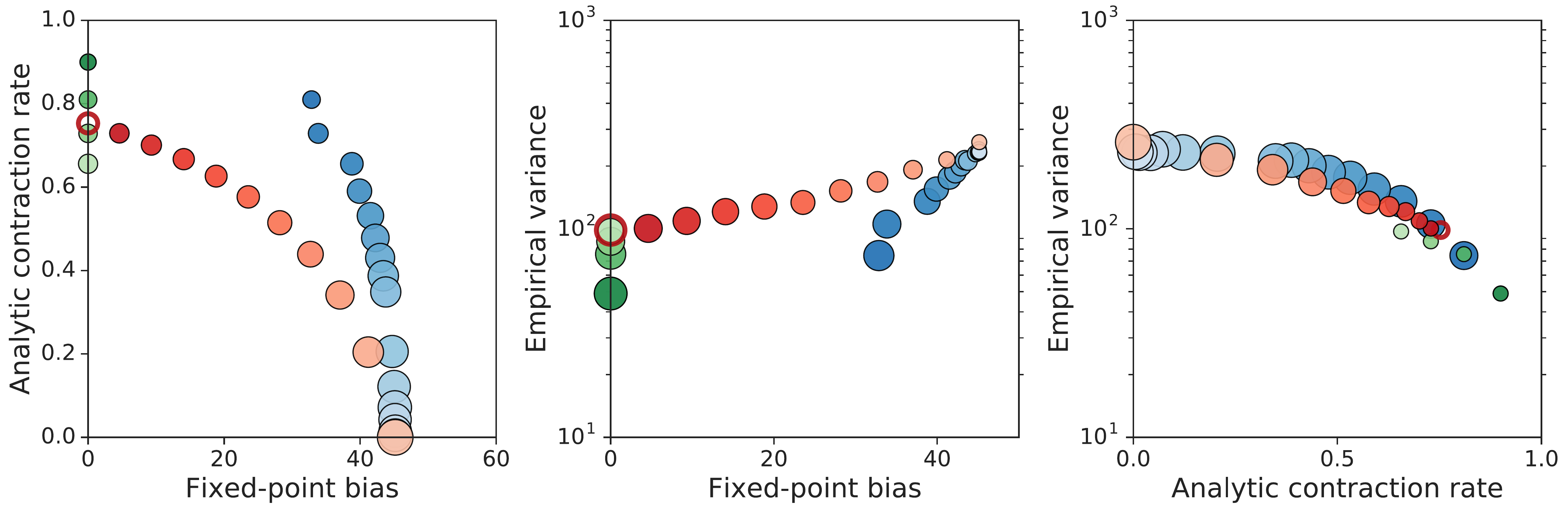}
      \vspace{-0.2cm}
      \caption{Target policy: uniform. Behaviour policy: Dirichlet($1,\ldots,1$) random.}
    \end{subfigure}%
    
    \begin{subfigure}{1.0\textwidth}
      \centering
      \includegraphics[keepaspectratio,width=.6\textwidth]{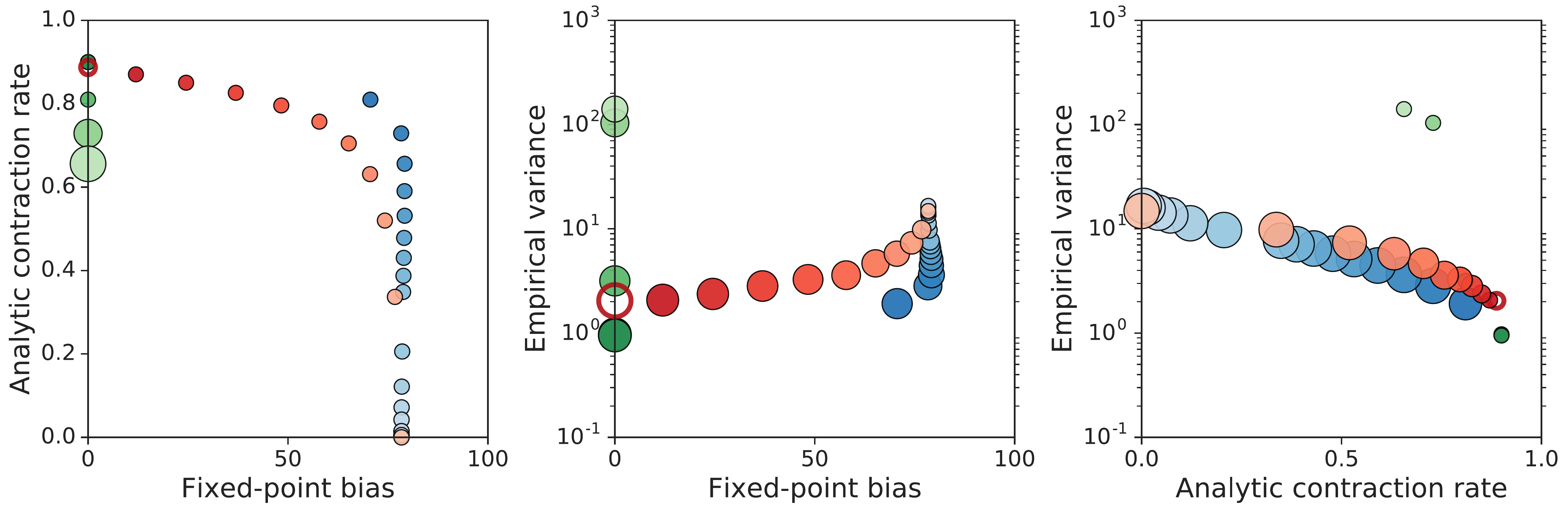}
      \vspace{-0.1cm}
      \caption{Target policy: Dirichlet($1,\ldots,1$) random. Behaviour policy: Independent Dirichlet($1,\ldots,1$) random.}
    \end{subfigure}%
    
    \begin{subfigure}{1.0\textwidth}
      \centering
      \includegraphics[keepaspectratio,width=.6\textwidth]{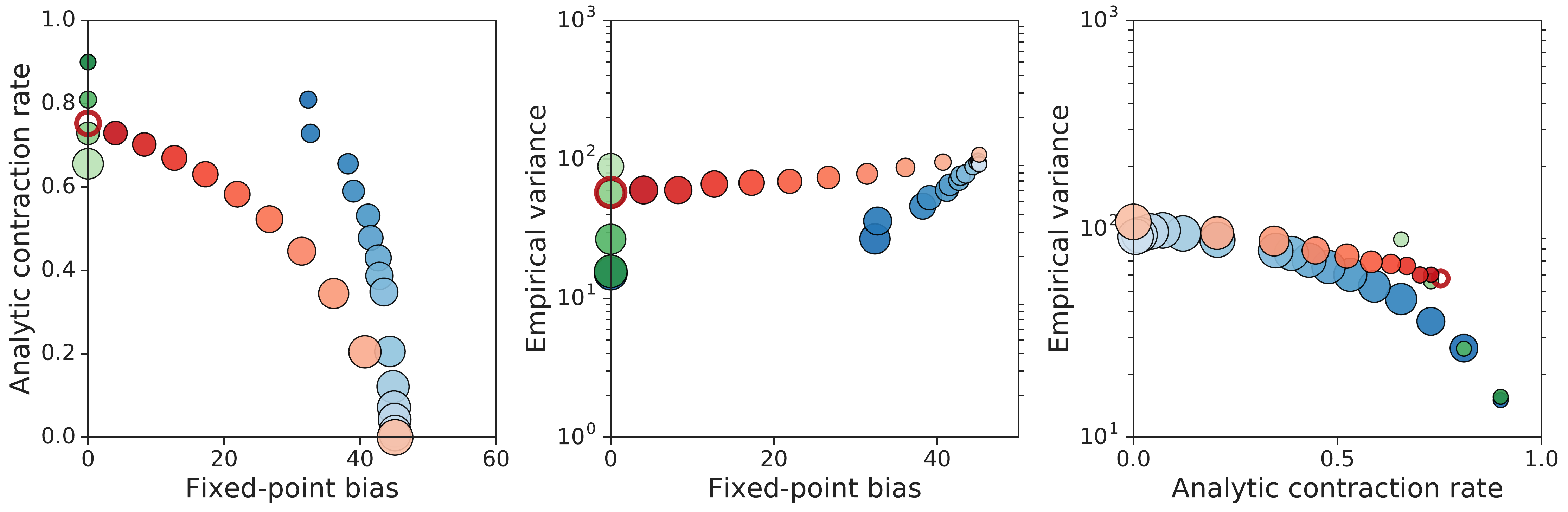}
      \vspace{-0.1cm}
      \caption{Target policy: Dirichlet($1,\ldots,1$) random. Behaviour policy: uniform.}
    \end{subfigure}%
    
    \begin{subfigure}{1.0\textwidth}
      \centering
      \includegraphics[keepaspectratio,width=.6\textwidth]{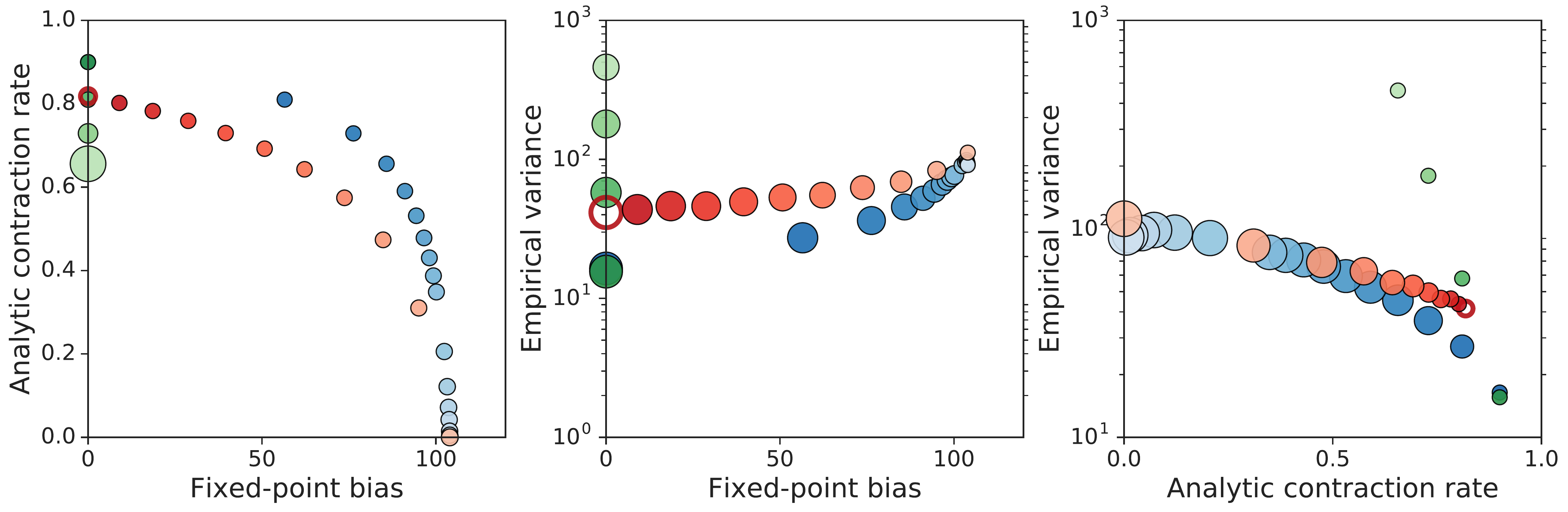}
      \vspace{-0.1cm}
      \caption{Target policy: optimal. Behaviour policy: uniform.}
    \end{subfigure}%
    
    \begin{subfigure}{1.0\textwidth}
      \centering
      \includegraphics[keepaspectratio,width=.6\textwidth]{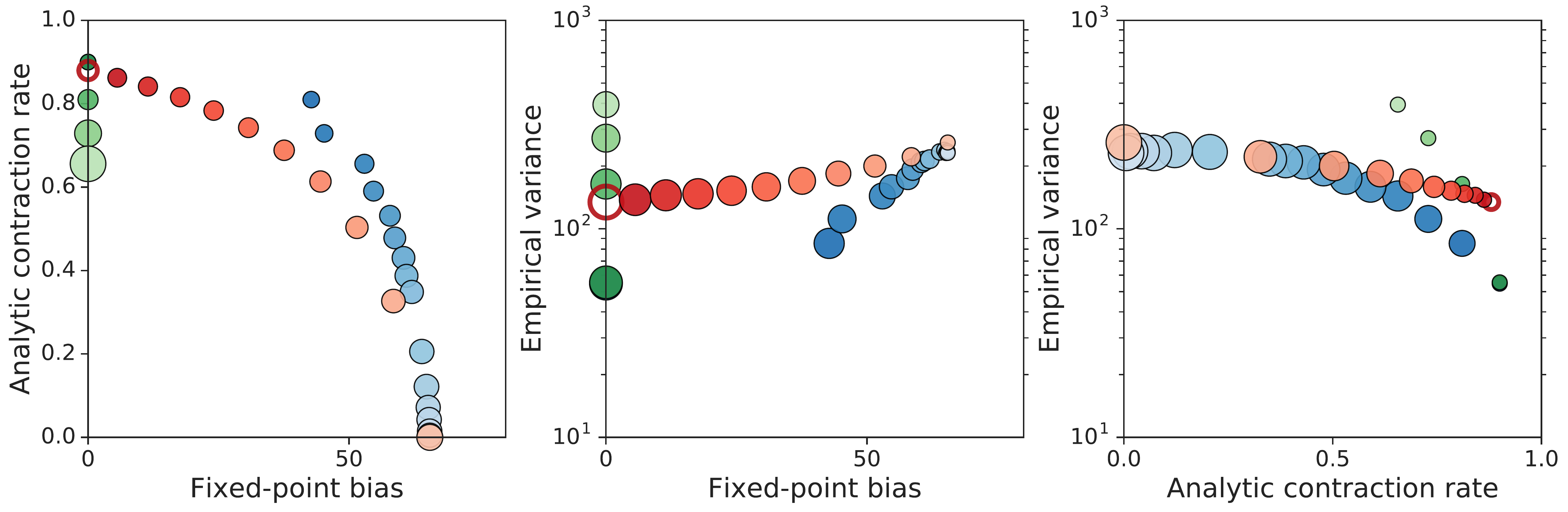}
      \vspace{-0.1cm}
      \caption{Target policy: optimal. Behaviour policy: Dirichlet($1,\ldots,1$) random.}
    \end{subfigure}%
    
    \begin{subfigure}{1.0\textwidth}
      \centering
      \includegraphics[keepaspectratio,width=.6\textwidth]{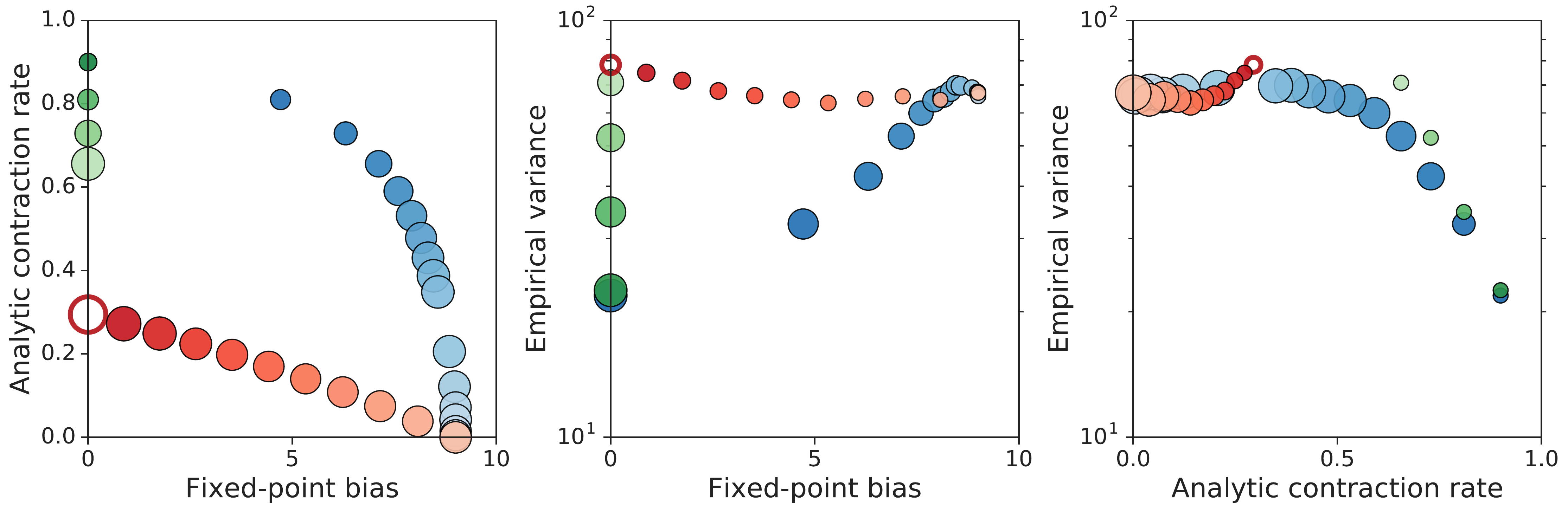}
      \vspace{-0.1cm}
      \caption{Target policy: optimal. Behaviour policy: optimal, with uniform exploration at probability $0.1$.}
    \end{subfigure}%
    
    \caption{Trade-off plots for the chain MDP described in Section \ref{sec:envs}, with $20$ states, with a variety of target policy/behaviour policy pairings.}
    \label{fig:pareto-gw1d}
\end{figure}

\clearpage
\section{Large-scale experiment details}

Episodes are limited to 30 minutes ($108,000$ environment frames). When reporting numeric scores, as opposed to learning curves, we give final agent performance as undiscounted episodic returns. The computing architectures used to run the two agents correspond precisely to the descriptions given in \cite{mnih2015human,kapturowski2018recurrent}.

Mini-batches are drawn from an experience replay buffer as described in the baseline agent papers \citep{mnih2015human,kapturowski2018recurrent}. For Retrace and \ctrace, the $n$-step loss function is modified to use the Retrace update for the, possibly modified, target policy. GPU training was performed on an NVIDIA Tesla V100.

Only for R2D2 experiments, all agents (including Retrace-based algorithms) use the invertible value function rescaling of R2D2. Finally, for \ctrace, the target policy is given by
\begin{equation*}
    \hat \pi := (1 - \alpha) \pi + \alpha \mu,
\end{equation*}
where $\pi$ is the greedy policy on the current action-values and $\mu$ is the $\epsilon$-greedy policy followed by the actor generating the current trajectory.
The value of $\alpha$ is adapted with each mini-batch using Robbins-Monro updates with truncated trajectory targets, as described in Section \eqref{sec:adaptingalpha}. The average observed contraction rate of Retrace over the mini-batch is calculated from the Retrace weights (see Equation~\eqref{eq:retracealphacontraction}),
\begin{equation*}
    \hat{\alpharetraceContract}(\alpha) = 1 - (1-\gamma) \sum_{t=0}^N \gamma^t \prod_{s=1}^t \left((1-\alpha) + \alpha \min\left(1, \frac{\pi(a_s|x_s)}{\mu(a_s|x_s)}\right) \right)  \, .
\end{equation*}
For simplicity we restate the Robbins-Monro update as a loss, scale up by $1000/(1 - \gamma)$ (to counter-act the small learning rate from Adam) and add it to the primary loss. We use a value of $\lambda = 1.0$ for all Retrace and \ctrace\ large-scale experiments. We considered $\lambda = 0.97$, in keeping with published work, but found larger values to perform better overall.

\subsection{R2D2 experiments}\label{sec:r2d2details}

\textbf{Network architecture.} R2D2, and our Retrace variants, use the 3-layer convolutional network from DQN \citep{mnih2015human}, followed by an LSTM with $512$ hidden units, which then feeds into a dueling architecture of size $512$ \citep{wang2016dueling}. Like the original R2D2, the LSTM receives the reward and one-hot action vector from the previous time step as inputs.

\textbf{Hyperparameters.} The hyperparameters used for the R2D2 agents follow those of \citet{kapturowski2018recurrent}, and are reproduced in Table \ref{tab:hyper} for completeness.

\begin{table}[h!]
    \centering
    \begin{tabular}{c|c}
         Number of actors &  $256$  \\
         Actor parameter update interval & $400$ environment steps \\
         \hline
         
         Sequence length $m$ & $80$ (+ prefix of $l = 40$ for burn-in) \\ 
         Replay buffer size & $4\times 10^6$ observations ($10^5$ part-overlapping sequences) \\
         Priority exponent & $0.9$ \\
         Importance sampling exponent & $0.6$ \\ 

         \hline
         Discount $\gamma$ & $0.997$ \\
         Minibatch size & $64$ \\
        
         Optimiser & Adam \citep{kingma2014adam} \\
         Optimiser settings & learning rate $= 10^{-4}$, $\varepsilon=10^{-3}$ \\
         Target network update interval & $2500$ updates \\
         Value function rescaling & $h(x) = \text{sign}(x)(\sqrt{|x| + 1} - 1) + \epsilon x, ~\epsilon=10^{-3}$\\

    \end{tabular}
    \caption{Hyperparameters values used in R2D2 experiments.}
    \label{tab:hyper}
\end{table}

\subsection{DQN experiments}\label{sec:dqndetails}

\textbf{Network architecture.} The DQN, DoubleDQN, $n$-step and Retrace-based agents use the 3-layer convolutional network from DQN \citep{mnih2015human}, but unlike the R2D2 agents do not use an LSTM or dueling architecture. Notice that the Retrace and \ctrace\ agents are effectively using DoubleDQN-style updates due to the target probabilities not coming from the target network.

\textbf{Hyperparameters.} For sequential DQN-agents ($n$-step and Retrace) we performed a preliminary hyperparameter sweep to determine appropriate learning rates for $n$-step and Retrace updates. We swept over learning rates ($0.00025, 0.0001, 0.00005, 0.00001$) for both algorithms, and for $n$-step we jointly swept over two values for $n$ ($3$ and $5$). These were run on four Atari 2600 games (Alien, Amidar, Assault, Asterix), with the best performing hyperparameters for each method used for the Atari-57 experiments.

Interestingly, we found a small learning rate of $0.00001$ worked best for both algorithms and that a larger $n = 5$ performed best for $n$-step. 

Both algorithms used a maximum sequence length of $16$. Due to shortness of the sequence length we use truncated trajectory corrections as described in the main text. Note that the truncation $\max(\Gamma, \gamma^N)$ is applied to each element of the sequence independently, therefore the value of $N$ will begin at $N = 16$ for the first element and reduce to $N = 1$ for the final transition in the replay sequence.

\clearpage

\section{Further large-scale results}

\subsection{Detailed R2D2 results}\label{sec:r2d2detailedresults}

We give further experimental results to complement the summary presented in the main paper; per-game training curves are given in Figure \ref{fig:r2d2-games}.

\begin{figure}[h]
    \centering
    \includegraphics[keepaspectratio,width=.94\textwidth]{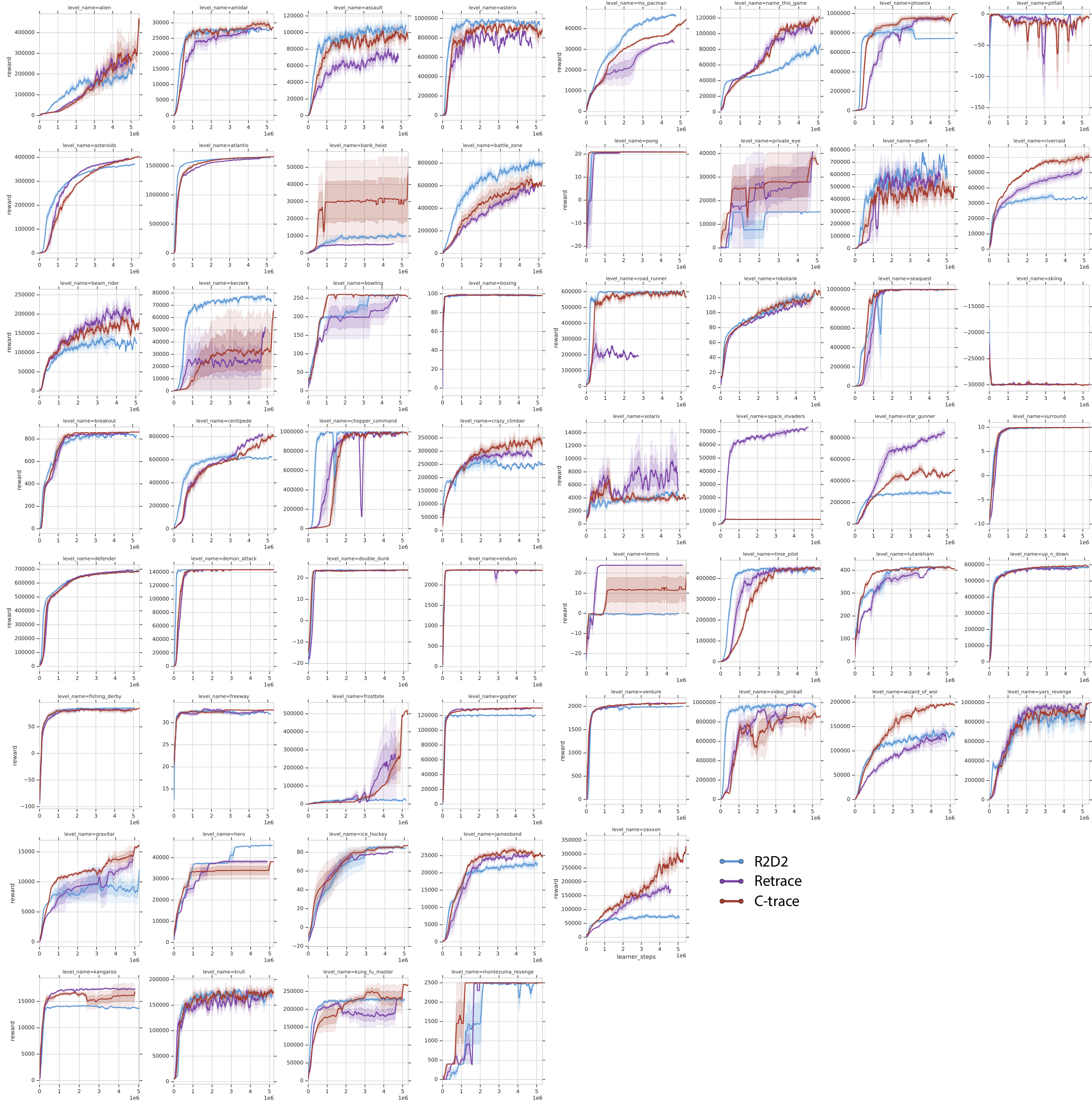}
	\caption{Training curves for 57 Atari games for R2D2 with $n$-step uncorrected returns (light blue), Retrace-R2D2 (black), and \ctrace-R2D2 (red).}
    \label{fig:r2d2-games}
\end{figure}

\clearpage

\subsection{Detailed DQN results}\label{sec:dqndetailedresults}

We give further experimental results to complement the summary presented in the main paper. Results for varying the contraction hyperparameter are given in Figure~\ref{fig:dqn_varied_n}, and per-game training curves for the main paper results are given in Figure \ref{fig:dqn-games}.

\begin{figure}[h]
    \centering
    \includegraphics[keepaspectratio,width=.9\textwidth]{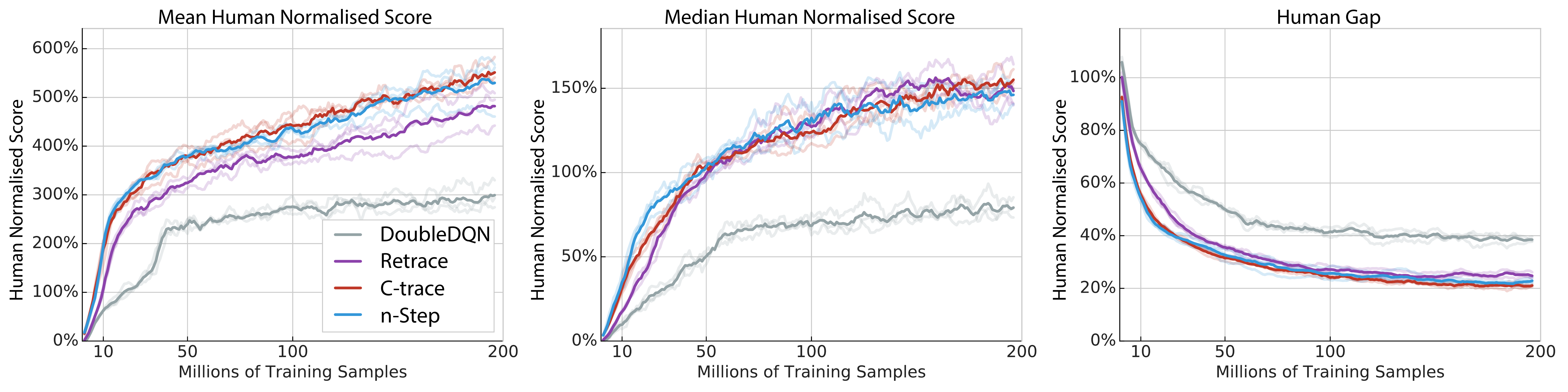}
    \includegraphics[keepaspectratio,width=.9\textwidth]{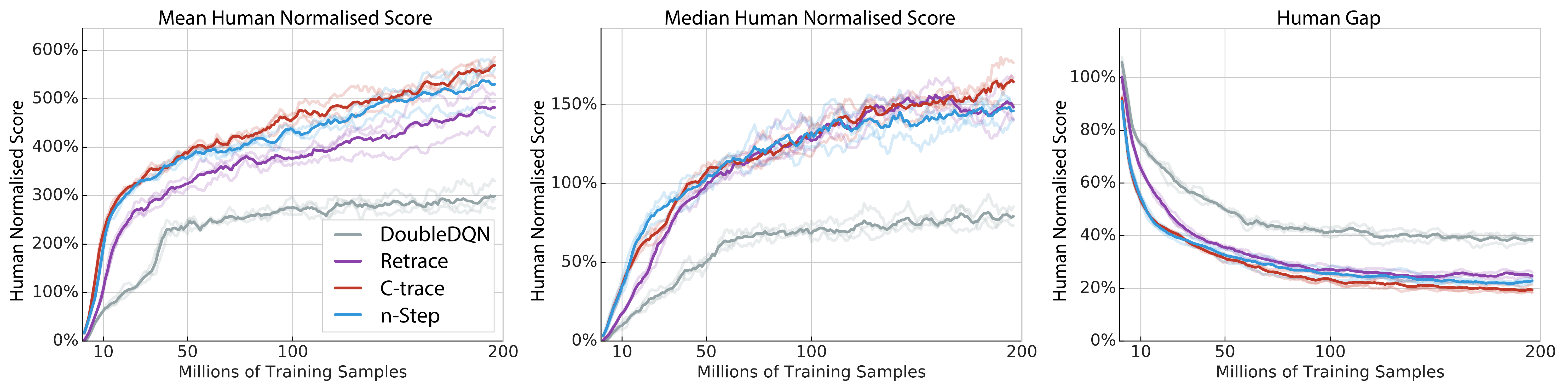}
    \includegraphics[keepaspectratio,width=.9\textwidth]{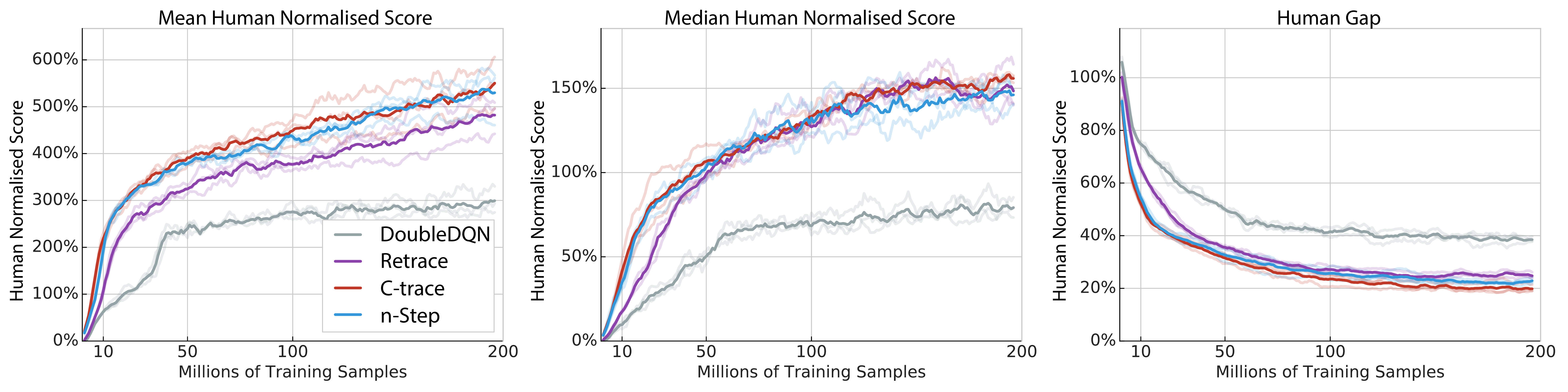}
    \caption{Atari-57 results for single-actor agent, as presented in the main text, but varying the C-trace contraction parameter: (\textbf{top}) $\gamma^5$, (\textbf{centre}) $\gamma^7$, and (\textbf{bottom}) $\gamma^{10}$. Notice that due to its adaptation of $\alpha$, C-trace is highly robust to the choice of contraction target.}
    \label{fig:dqn_varied_n}
\end{figure}

\begin{figure}[h]
    \centering
    \includegraphics[keepaspectratio,width=.85\textwidth]{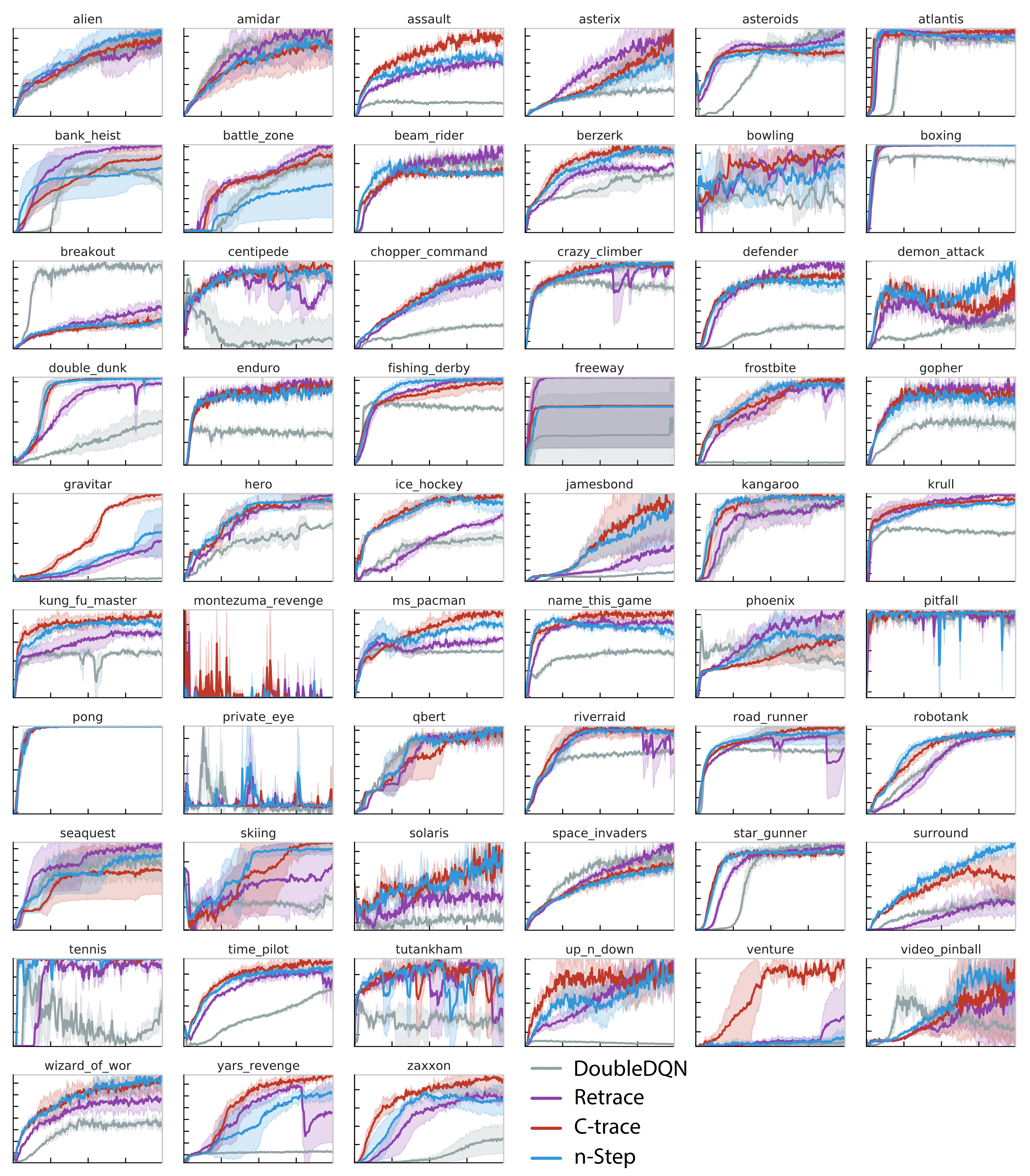}
    \caption{Training curves for 57 Atari games for Double DQN (grey), Double DQN with $n$-step uncorrected returns (light blue), Retrace-DQN (black), and \ctrace-DQN (red).}
    \label{fig:dqn-games}
\end{figure}

\subsection{Empirical verification of C-trace contraction rate}\label{sec:contraction-verification}

In this section, we provide additional analysis of the empirical R2D2 results described in the main paper, to investigate whether C-trace is able to target the desired contraction rate in practice when used in combination with a deep reinforcement learning architecture. We compute averaged contraction rates using Equation \eqref{eq:retracealphacontraction}, and in Figure~\ref{fig:contraction_dists} we provide kernel density estimation (KDE) plots of these  contraction rates achieved by C-trace and Retrace at the end of training across two classes of games. Firstly, those games for which Retrace achieves a contraction rate of more than $\gamma^{10}$, the chosen target rate for C-trace in this instance, and secondly, the complement of these games. Splitting the games into these two classes illustrates the behaviour of C-trace clearly. In the first class of games, it is possible for C-trace to attain the contraction rate $\gamma^{10}$, by lowering the mixture parameter, whilst in the second class, Retrace already has a contraction rate below the target of $\gamma^{10}$. This description is reflected in the plots; in the left plot of Figure~\ref{fig:contraction_dists}, the Retrace distribution lies to the right of the target contraction rate, whilst the C-trace distribution is centred precisely on this rate, whilst on the right, the Retrace distribution lies to the left of the target contraction rate, and the C-trace distribution closely matches it, as there is no possibility of increasing the target contraction rate. We also provide per-game KDE plots of contraction rates attained throughout training in Figure~\ref{fig:per_game_contraction}.

\begin{figure}
    \centering
    \null
    \hfill
    \includegraphics[keepaspectratio,width=.35\textwidth]{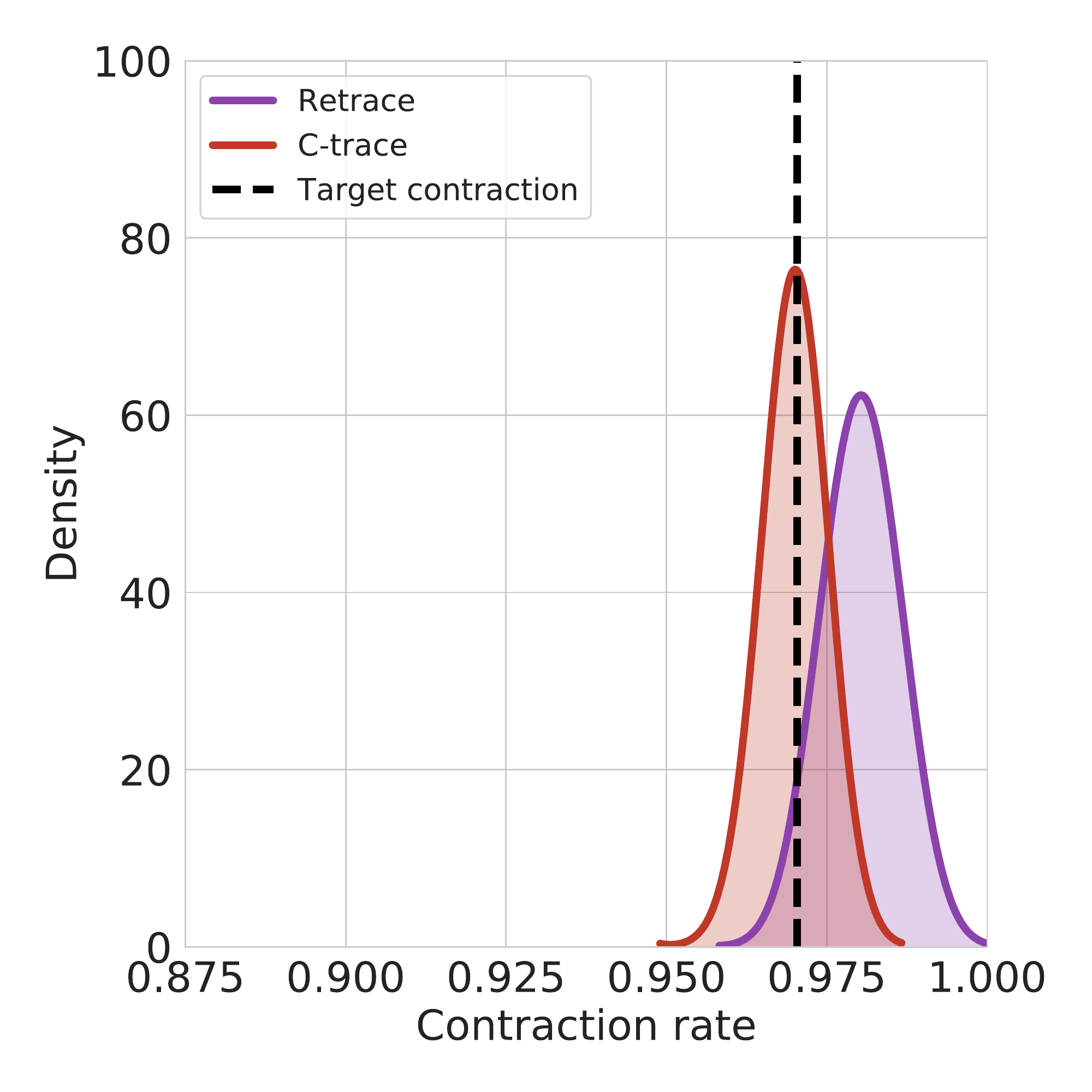}
    \hfill
    \includegraphics[keepaspectratio,width=.35\textwidth]{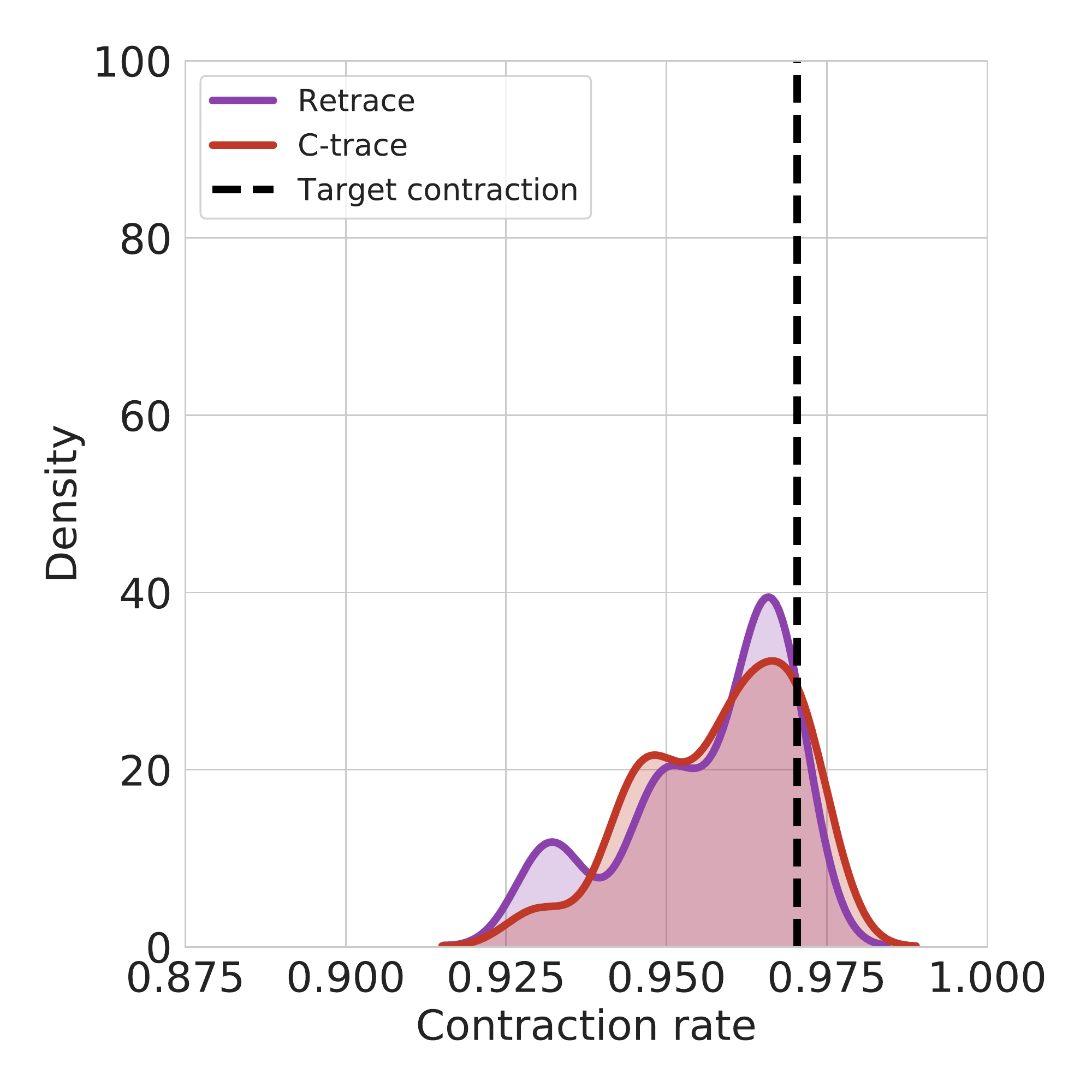}
    \hfill
    \null
    \caption{Contraction rates achieved by Retrace and C-trace at the end of training, across two classes of games. Left: games in which the contraction rate of Retrace is greater than the C-trace target of $\gamma^{10}$. Right: games for which the contraction rate of Retrace is less than the C-trace target of $\gamma^{10}$.}
    \label{fig:contraction_dists}
\end{figure}

\begin{figure}
    \centering
    \includegraphics[keepaspectratio, height=.95\textheight]{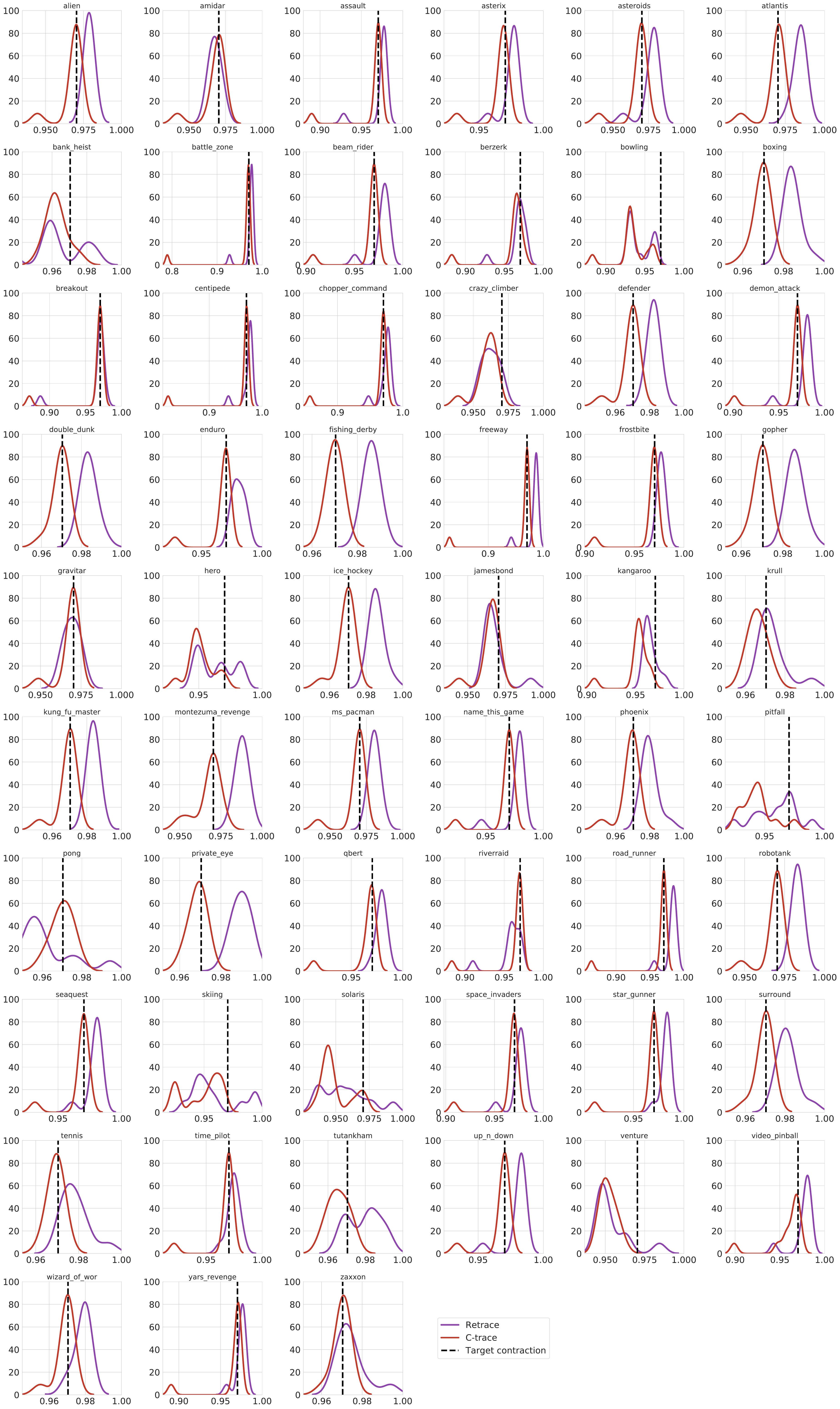}
    \caption{Per-game contraction rates for Retrace and C-trace throughout training.}
    \label{fig:per_game_contraction}
\end{figure}

\end{document}